\documentclass{article}

\usepackage{microtype}
\usepackage{graphicx}
\usepackage{booktabs} 

\usepackage{hyperref}

\usepackage[accepted]{icml2024}

\usepackage{amsmath}
\usepackage{amssymb}
\usepackage{mathtools}
\usepackage{amsthm}

\usepackage[capitalize,noabbrev]{cleveref}

\theoremstyle{plain}
\newtheorem{theorem}{Theorem}[section]
\newtheorem{proposition}[theorem]{Proposition}
\newtheorem{lemma}[theorem]{Lemma}

\theoremstyle{definition}
\newtheorem{definition}[theorem]{Definition}
\newtheorem{assumption}[theorem]{Assumption}
\theoremstyle{remark}
\newtheorem{remark}[theorem]{Remark}

\usepackage[textsize=tiny]{todonotes}

\usepackage{thmtools, mathtools,mathabx}
\usepackage{algorithm}
\usepackage{algorithmic}
\usepackage{subcaption}
\usepackage{paralist}
\usepackage{cleveref}

\usepackage{etoolbox}
\newcommand{\zerodisplayskips}{%
  \setlength{\abovedisplayskip}{4.5pt}%
  \setlength{\belowdisplayskip}{4.5pt}%
  \setlength{\abovedisplayshortskip}{4.5pt}%
  \setlength{\belowdisplayshortskip}{4.5pt}}
\appto{\normalsize}{\zerodisplayskips}
\appto{\small}{\zerodisplayskips}
\appto{\footnotesize}{\zerodisplayskips}

\DeclareMathOperator*{\argmax}{arg\,max}
\DeclareMathOperator*{\argmin}{arg\,min}
\DeclareMathOperator*{\opsE}{\mathbb{E}}
\DeclareMathOperator{\E}{\mathbb{E}}
\let\P\relax
\DeclareMathOperator{\P}{\mathbb{P}}
\let\KL\relax
\DeclareMathOperator{\KL}{KL}
\DeclareMathOperator{\TV}{TV}

\DeclareMathOperator{\Normal}{Normal}
\DeclareMathOperator{\Uniform}{Uniform}

\newcommand\numberthis{\addtocounter{equation}{1}\tag{\theequation}}

\newcommand{\norm}[1]{\left\lVert#1\right\rVert}
\newcommand{\regnorm}[1]{\lVert#1\rVert}
\newcommand{\bignorm}[1]{\big\lVert#1\big\rVert}

\newcommand{\indep}{\perp \!\!\! \perp}

\newcommand{\N}{\mathbb{N}}
\newcommand{\R}{\mathbb{R}}
\newcommand{\calA}{\mathcal{A}}
\newcommand{\calB}{\mathcal{B}}
\newcommand{\calC}{\mathcal{C}}
\newcommand{\calD}{\mathcal{D}}

\newcommand{\calF}{\mathcal{F}}

\newcommand{\calH}{\mathcal{H}}

\newcommand{\calO}{\mathcal{O}}
\newcommand{\calP}{\mathcal{P}}
\newcommand{\calQ}{\mathcal{Q}}

\newcommand{\calS}{\mathcal{S}}

\newcommand{\calU}{\mathcal{U}}
\newcommand{\calV}{\mathcal{V}}

\newcommand{\calX}{\mathcal{X}}
\newcommand{\calY}{\mathcal{Y}}
\newcommand{\calZ}{\mathcal{Z}}

\newcommand{\frakR}{\mathfrak{R}}
\newcommand{\frakRwn}{\mathfrak{R}_{\w,\nC}}
\newcommand{\frakRwnF}{\frakR_{\w,\nC}(\calF_\Pi)}

\newcommand{\eps}{\varepsilon}

\newcommand{\Ri}{R_c}
\newcommand{\Qi}{Q_c}
\newcommand{\Qtilde}{\tilde Q}
\newcommand{\Qhat}{\hat Q}

\newcommand{\ohat}{\hat{w}}

\newcommand{\kij}{k_c(i)}

\newcommand{\muhatikij}{\hat{\mu}_c^{-k_c(i)}}

\newcommand{\ohatikij}{\ohat_c^{-k_c(i)}}

\newcommand{\Di}{\calD_c}
\newcommand{\calDc}{\calD_c}

\newcommand{\tildecalDc}{\tilde{\calD}_c}
\newcommand{\calDcec}{\bar{\calD}_c}
\newcommand{\calDw}{\calD_\w}
\newcommand{\calDwew}{\bar{\calD}_\w}

\newcommand{\calDkek}{\bar{\calD}_k}

\newcommand{\calPc}{\calP_c}

\newcommand{\mui}{\mu_c}
\newcommand{\ei}{e_c}
\newcommand{\oi}{w_c}
\renewcommand{\ni}{n_c}
\renewcommand{\Xi}{X^c}
\newcommand{\Ai}{A^c}
\newcommand{\Wi}{W^c}
\newcommand{\Yi}{Y^c}
\newcommand{\vecYi}{\vec{Y}^c}
\newcommand{\Zi}{Z^c}

\newcommand{\Gi}{\Gamma^c}

\newcommand{\Xij}{X_{i}^{c}}
\newcommand{\Aij}{A_{i}^{c}}
\newcommand{\Wij}{W_{i}^{c}}
\newcommand{\Wijprime}{W_{i}^{c\prime}}
\newcommand{\Yij}{Y_{i}^{c}}
\newcommand{\Zij}{Z_{i}^{c}}

\newcommand{\Whatij}{\hat{W}_{i}^{c}}

\newcommand{\Gij}{\Gamma_{i}^{c}}
\newcommand{\Gijp}{\Gamma_{i}^{c\prime}}
\newcommand{\Gijpp}{\Gamma_{i}^{c\prime\prime}}
\newcommand{\Gijppp}{\Gamma_{i}^{c\prime\prime\prime}}
\newcommand{\Ghatij}{\hat{\Gamma}_{i}^{c}}
\newcommand{\vecGij}{\vec{\Gamma}_{i}^{c}}
\newcommand{\epsij}{\varepsilon_{i}^{c}}

\newcommand{\Xk}{X^k}
\newcommand{\ek}{e_k}

\newcommand{\gij}{\gamma_{i}^{c}}
\newcommand{\omegaij}{\omega_{i}^{c}}
\newcommand{\xij}{x_{i}^{c}}

\newcommand{\x}{\{\xij\mid c\in\calC, i\in[n_c]\}}

\newcommand{\alphaK}{\alpha_K}
\newcommand{\zetamu}{{\zeta_\mu}}

\newcommand{\zetao}{{\zeta_w}}
\newcommand{\nuc}{\nu_{c}}
\newcommand{\gc}{g_c}

\newcommand{\tilden}{\tilde{n}}

\newcommand{\tildenij}{\tilde{n}_{i}^{c}}
\newcommand{\tildex}{\tilde{x}}
\newcommand{\tildexij}{\tilde{x}_{i}^{c}}

\newcommand{\ones}{\mathbf{1}}

\newcommand{\pihat}{\hat\pi}
\newcommand{\pihatw}{\hat{\pi}_\w}
\newcommand{\pistari}{\pi^*_c}
\newcommand{\pistarw}{\pi^*_\w}

\newcommand{\weight}{\lambda}
\newcommand{\w}{\weight}
\newcommand{\wi}{\weight_c}
\newcommand{\Qw}{Q_\w}
\newcommand{\Qhatw}{\hat{Q}_\w}
\newcommand{\Qtildew}{\tilde{Q}_\w}
\newcommand{\Rw}{R_\w}

\newcommand{\muhati}{\hat{\mu}_c}

\newcommand{\ohati}{\hat{w}_c}

\newcommand{\Qhati}{\hat{Q}_c}

\newcommand{\Deltaw}{\Delta_\w}
\newcommand{\Deltahatw}{\hat{\Delta}_\w}
\newcommand{\Deltatildew}{\tilde{\Delta}_\w}

\newcommand{\nc}{n_{c}}
\newcommand{\nC}{n_{\calC}}
\newcommand{\barn}{\bar{n}}
\newcommand{\barni}{\barn_c}

\newcommand{\barV}{\bar{V}}

\newcommand{\supF}{\sup_{f\in\calF}}
\newcommand{\supH}{\sup_{h\in\calH}}

\newcommand{\supPi}{\sup_{\pi\in\Pi}}
\newcommand{\supPiab}{\sup_{\pi_a,\pi_b\in\Pi}}

\newcommand{\maxPi}{\max_{\pi\in\Pi}}
\newcommand{\maxC}{\max_{c\in\calC}}
\newcommand{\minC}{\min_{c\in\calC}}

\newcommand{\pia}{\pi_a}
\newcommand{\pib}{\pi_b}

\newcommand{\sumin}{\sum_{i=1}^{n}}
\newcommand{\sumaA}{\sum_{a\in\calA}}
\newcommand{\sumkK}{\sum_{k=1}^K}
\newcommand{\sumiM}{\sum_{c\in\calC}}
\newcommand{\fracwini}{\frac{\w_c}{n_c}}

\newcommand{\fracwisqni}{\frac{\w_c^2}{n_c}}
\newcommand{\fracwinisq}{\frac{\w_c^2}{n_c^2}}

\newcommand{\fracwibarni}{\frac{\w_c}{\barn_c}}
\newcommand{\fracwisqbarni}{\frac{\w_c^2}{\barn_c}}
\newcommand{\fracskewnessn}{\frac{\skewness}{n}}
\newcommand{\sqrtfracskewnessn}{\sqrt{\fracskewnessn}}

\newcommand{\fracVwn}{\frac{\Vwn}{n}}
\newcommand{\sqrtfracVwn}{\sqrt{\fracVwn}}

\newcommand{\sumjni}{\sum_{i=1}^{n_c}}
\newcommand{\sumijw}{\sumiM\fracwini\sumjni}
\newcommand{\sumijwsq}{\sumiM\sumjni\fracwinisq}
\newcommand{\sumij}{\sumiM\sumjni}

\newcommand{\swnC}{s_{\w,\nC}}
\newcommand{\UwnC}{U_{\w,\nC}}
\newcommand{\skewness}{\mathfrak{s}(\w\|\barn)}

\newcommand{\Bw}{B_{\w,\nC}}
\newcommand{\BwW}{B_{\w,\nC}(W)}

\newcommand{\Vwn}{V_{\w,\nC}}
\newcommand{\Uwn}{U_{\w,\barn}}

\newcommand{\regabs}[1]{|#1|}
\newcommand{\abs}[1]{\left|#1\right|}
\newcommand{\regbra}[1]{[#1]}
\newcommand{\bra}[1]{\left[#1\right]}

\newcommand{\rbra}[1]{\left(#1\right)}

\newcommand{\bigabs}[1]{\big|#1\big|}
\newcommand{\Bigabs}[1]{\Big|#1\Big|}
\newcommand{\biggabs}[1]{\bigg|#1\bigg|}
\newcommand{\Biggabs}[1]{\Bigg|#1\Bigg|}
\newcommand{\smallbra}[1]{[#1]}
\newcommand{\bigbra}[1]{\big[#1\big]}
\newcommand{\Bigbra}[1]{\Big[#1\Big]}
\newcommand{\biggbra}[1]{\bigg[#1\bigg]}

\newcommand{\bigpar}[1]{\big(#1\big)}
\newcommand{\Bigpar}[1]{\Big(#1\Big)}
\newcommand{\biggpar}[1]{\bigg(#1\bigg)}

\newcommand{\Bigmid}{\ \Big|\ }

\newcommand{\ellw}{\ell_{\w,2}}

\newcommand{\Ham}{\text{H}}

\newcommand{\VC}{\text{VC}}

\newcommand{\littleo}[1]{o\left(#1\,\right)}
\newcommand{\littleop}[1]{o_p\left(#1\,\right)}
\newcommand{\bigO}[1]{\calO\left(#1\,\right)}

\icmltitlerunning{Federated Offline Policy Learning}

\begin{document}

\twocolumn[
\icmltitle{Federated Offline Policy Learning}



\icmlsetsymbol{equal}{*}




\begin{icmlauthorlist}
\icmlauthor{Aldo Gael Carranza}{sta}
\icmlauthor{Susan Athey}{sta}
\end{icmlauthorlist}

\icmlaffiliation{sta}{Stanford University}

\icmlcorrespondingauthor{Aldo Gael Carranza}{aldogael@stanford.edu}

\icmlkeywords{Machine Learning, ICML, Offline Policy Learning, Contextual Bandits, Federated Learning}

\vskip 0.3in
]



\printAffiliationsAndNotice{}  

\begin{abstract}
We consider the problem of learning personalized decision policies from observational bandit feedback data across multiple heterogeneous data sources. In our approach, we introduce a novel regret analysis that establishes finite-sample upper bounds on distinguishing notions of global regret for all data sources on aggregate and of local regret for any given data source. We characterize these regret bounds by expressions of source heterogeneity and distribution shift. Moreover, we examine the practical considerations of this problem in the federated setting where a central server aims to train a policy on data distributed across the heterogeneous sources without collecting any of their raw data. We present a policy learning algorithm amenable to federation based on the aggregation of local policies trained with doubly robust offline policy evaluation strategies. Our analysis and supporting experimental results provide insights into tradeoffs
in the participation of heterogeneous data sources in offline policy learning.
\end{abstract}

\section{Introduction}

Offline policy learning from observational bandit feedback data
is an effective approach for learning personalized decision policies in applications where obtaining online, real-time data is impractical
\citep{swaminathan2015batch, kitagawa2018should,athey2021policy}. Typically, the observational data used in offline policy learning is assumed to originate from a single source distribution. However, in practice, we often have the opportunity to leverage multiple datasets collected from various experiments
under different populations, environments, or logging policies \citep{kallus2021optimal}.
For instance, a healthcare policymaker may have access to data from multiple hospitals that conduct different types of clinical trials on distinct patient populations.
Learning from multiple heterogeneous observational datasets,
with their more diverse and extensive coverage of the decision space, may lead to better personalized decision policies, assuming sufficient similarity and generalization across data sources.

However, several practical constraints, such as privacy concerns, legal restrictions, proprietary interests, or competitive barriers, can hinder the consolidation of datasets across sources.
Federated learning \citep{kairouz2021advances}
presents a potential solution to such obstacles by offering
a framework for training machine learning models in a decentralized manner, thereby minimizing systemic risks
associated with traditional, centralized machine learning.
Although federated learning principles have been widely implemented on standard supervised learning tasks, they have not been extensively explored in more prescriptive machine learning tasks such as policy learning.
Federated learning techniques applied to policy learning can enable platforms to learn targeted decision policies without centrally storing sensitive user information. It also has the potential to incentivize institutions to collaborate on developing policies that are more generalizable across diverse environments without having to share sensitive data, such as clinical patient data in hospitals.

In this work, we introduce the problem of learning personalized decision policies on observational bandit feedback data from multiple heterogeneous data sources.
Our main technical contribution is the presentation of a novel regret analysis that establishes finite-sample upper bounds on distinguishing notions of \textit{global regret} for all data sources on aggregate and of \textit{local regret} for any given data source.
We characterize these regret bounds by expressions of source heterogeneity and distribution shift.
Moreover, we consider the additional complexities of this problem in the federated setting where central data aggregation across sources is constrained and data source heterogeneity is known to be a significant practical concern.
We introduce an offline policy learning algorithm amenable to federation and we experimentally verify the effect of source heterogeneity on regret. Additionally, we present design choices to overcome local performance degradation due to distribution shift. Our analysis and supporting experimental results in this work help provide insights into the tradeoffs in the participation of heterogeneous data sources in offline policy learning.

\section{Related Work}

\paragraph{Offline Policy Learning}

In recent years, there have been significant advancements in offline policy learning from observational bandit feedback data.
\citet{swaminathan2015batch, kitagawa2018should} established frameworks for structured decision policy learning using offline policy evaluation strategies. \citet{athey2021policy} achieved optimal regret rates under unknown propensities through doubly robust estimators. \citet{zhou2023offline} extended these results to the multi-action setting. \citet{kallus2018balanced} found optimal weights for the target policy directly from the data. \citet{zhan2021policy} ensure optimal regret guarantees under adaptively collected data with diminishing propensities. \citet{jin2022policy} relaxed the uniform overlap assumption to partial overlap under the optimal policy.
We also mention that many contextual bandit methods often utilize offline policy learning oracles when devising adaptive action-assignment rules \citep{bietti2021contextual,krishnamurthy2021adapting, simchi2022bypassing, carranza2022flexible}.
More relevant to our setting
of policy learning under heterogeneous data sources,
\citet{agarwal2017effective,he2019off,kallus2021optimal} leveraged data from multiple historical logging policies, but they assume the same underlying populations and environments. 
Our problem setting is also very closely related to multi-task offline policy learning, where each task can be seen as a data source.
\citet{hong2023multi} study a hierarchical Bayesian approach to this problem where task similarity is captured using an unknown latent parameter.
In our experimental section, we further discuss this method and compare results of this method with those of our approach.

\vspace{-.75em}
\paragraph{Federated Learning}

\citet{kairouz2021advances, wang2019federated} offer comprehensive surveys on federated learning and its challenges. \citet{mohri2019agnostic} presented an agnostic supervised federated learning framework, introducing multiple-source learning concepts like weighted Rademacher complexity and skewness measures which we make use of in our work. \citet{wei2021federated} established excess risk bounds for supervised federated learning under data heterogeneity. \citet{li2019convergence,li2020federated,karimireddy2020scaffold} explore the impact of client heterogeneity on model convergence in federated learning. Contextual bandits in federated settings have been studied by \citet{agarwal2020federated,dubey2020differentially,huang2021federated,agarwal2023empirical}. However, offline policy evaluation and learning in federated settings remain relatively underexplored.  \citet{xiong2021federated} investigated federated methods for estimating average treatment effects across heterogeneous data sources. \citet{zhou2022federated,shen2023distributed} delved into federated offline policy optimization in full reinforcement learning settings but with significant limitations,
including relying on strong linear functional form assumptions with highly suboptimal rates or a difficult saddle point optimization approach focusing more on policy convergence.

\vspace{-1em}
\section{Preliminaries}

\subsection{Setting}\label{sec:Preliminaries-Setting}

We introduce the problem of offline policy learning from observational bandit feedback data across multiple heterogeneous data sources. Throughout the paper, we refer to a heterogeneous data source as a \textit{client} and the central planner that aggregates client models as the \textit{central server}, following standard federated learning terminology.

Let $\calX\subset\R^p$ be the context space, $\calA=\{a_1,\dots,a_d\}$ be the finite action space with $d$ actions, and $\calY\subset\R$ be the reward space. A \textit{decision policy} $\pi:\calX\to\calA$ is a mapping from the context space $\calX$ to actions $\calA$. We assume there is a central server and a finite set of clients $\calC$, with each client $c\in\calC$ possessing a \textit{local data-generating distribution} $\Di$ defined over $\calX\times\calY^d$ which governs how client contexts $\Xi$ and client potential reward outcomes $\Yi(a_1),\dots,\Yi(a_d)$ are generated. Moreover, the central server specifies a fixed distribution $\w$ over the set of clients $\calC$ describing how clients will be sampled or aggregated\footnote{Clients are sampled in the cross-device FL setting and aggregated in the cross-silo FL settings.}, which we will simply refer to as the \textit{client sampling distribution}.

The central server seeks to train a decision policy that performs well on the \textit{global data-generating mixture distribution} defined by $\calDw\coloneqq\sumiM\wi\calDc$. At the same time, if there is a potential target client of interest, the central server may not want the personalized policy to perform poorly on the local distribution of this client, otherwise their locally trained policy may provide greater utility to the client, and thus their participation is disincentivized.
In the following section, we introduce the exact policy performance measures that capture these two potentially opposing objectives.

\subsection{Objective}

We consider the immediate reward gained by a client by taking actions according to any given policy. Additionally, we extend this metric to a global version that captures the aggregate reward gained from the mixture of clients under the client sampling distribution.

\begin{definition}
    The \textit{local policy value} under client $c$ and the \textit{global policy value} under client sampling distribution $\w$ of a policy $\pi$ are, respectively,
    \begin{gather*}
        \Qi(\pi)\coloneqq\opsE_{\Zi\sim\calDc}\smallbra{Y^c(\pi(X^c))} \\
        \Qw(\pi)\coloneqq\opsE_{c\sim\w\vphantom{\Zi\sim\calDc}}\opsE_{\vphantom{c\sim\w}\Zi\sim\calDc}\regbra{\Yi(\pi(\Xi))},
    \end{gather*}
    where the expectations are taken with respect to the corresponding local data-generating distributions $\Zi=(\Xi,\Yi(a_1),\dots,\Yi(a_d))\sim\calDc$ and the client sampling distribution $c\sim\w$.
\end{definition}

The performance of a policy is typically characterized by a notion of regret against an optimal policy in a specified \textit{policy class} $\Pi\subset\{\pi:\calX\to\calA\}$, which we assume to be fixed throughout the paper. We define local and global versions of regret based on their respective versions of policy values.

\begin{definition}
    The \textit{local regret} under client $c$ and the \textit{global regret} under client sampling distribution $\w$ of a policy $\pi$ relative to the given policy class $\Pi$ are, respectively,
    \begin{gather*}
        \Ri(\pi)\coloneqq\max_{\pi'\in\Pi}\Qi(\pi')-\Qi(\pi) \\
        \Rw(\pi)\coloneqq\max_{\pi'\in\Pi}\Qw(\pi')-\Qw(\pi).
    \end{gather*}
\end{definition}

The objective of the central server is to determine a policy in the specified policy class $\Pi$ that minimizes global regret. On the other hand, the central server also aims to characterize the corresponding local regret of a target client under the obtained policy since this quantity captures the client's corresponding individual utility to a global policy.

\subsection{Data-Generating Processes}\label{sec:Preliminaries-Data}

We assume each client $c\in\calC$ has a \textit{local observational data set}
$\{(\Xij,\Aij,\Yij)\}_{i=1}^{n_c}\subset\calX\times\calA\times\calY$ consisting of $n_c\in\N$ triples of contexts, actions, and rewards collected using a \textit{local experimental stochastic policy} $e_c:\calX\to\Delta(\calA)$ in the following manner. For the $i$-th data point of client $c$,
\vspace{-.25em}
\begin{compactenum}
    \item nature samples
    $(\Xij,\Yij(a_1),\dots,\Yij(a_d))\sim\calDc$;

    \item client $c$ is assigned action $\Aij\sim e_c(\cdot|\Xij)$;
    
    \item client $c$ observes the realized outcome $\Yij=\Yij(\Aij)$ ;

    \item client $c$ logs the data tuple $(\Xij,\Aij,\Yij)$ locally.\footnote{If the propensity $e_c(\Aij|\Xij)=\P(\Aij|\Xij)$ is known, it also locally logged as it facilitates subsequent policy value estimation.}
\end{compactenum}

We will let $\smash[b]{n\coloneqq\sumiM\ni}$ denote the \textit{total sample size} across clients.
Note that although the counterfactual reward outcomes $\Yij(a)$ for all $a\in\calA\backslash\{\Aij\}$ exist in the local data-generating process, they are not observed in the realized data. All clients only observe the outcomes associated to their assigned treatments. For this reason, such observational data is also referred to as \textit{bandit feedback data} \citep{swaminathan2015batch}.

Given these data-generating processes, it will be useful to introduce the data-generating distributions that also incorporate how actions are sampled. 
For each client $c\in\calC$, the local historical policy $e_c$ induces
a \textit{complete local data-generating distribution}
$\calDcec$ defined over $\calX\times\calA\times\calY^d$
that dictates how the entire local contexts, actions, and potential outcomes were sampled in the local data-generating process, i.e., $(\Xij,\Aij,\Yij(a_1),\dots,\Yij(a_d))\sim\calDcec$.
Given this construction of the complete client distributions, we also introduce the \textit{complete global data-generating mixture distribution} defined by $\calDwew\coloneqq\sumiM\wi\calDcec$.

\subsection{Data Assumptions}\label{sec:Preliminaries-DataAssumptions}
We make the following standard assumptions on the data-generating process of any given client.
\begin{assumption}[Local Ignorability]\label{ass:dgp}
    For any client $c\in\calC$, the complete local data-generating distribution $(\Xi,\Ai,\Yi(a_1),\dots,\Yi(a_d))\sim\calDcec$ satisfies:
    \vspace{-.25em}
    \begin{compactenum}
        \item[(a)] \textit{Boundedness}: The marginal distribution of $\calDcec$ on the set of potential outcomes $\calY^d$ has bounded support, i.e., there exists some $B_c>0$ such that $\abs{\Yi(a)}\le B_c$ for all $a\in\calA$.
        
        \item[(b)] \textit{Unconfoundedness}:
        Potential outcomes are independent of the observed action conditional on the observed context, i.e., $(\Yi(a_1),\dots,\Yi(a_d))\indep \Ai\mid\Xi$.

        \item[(c)] \textit{Overlap}: For any given context, every action has a non-zero probability of being sampled, i.e., there exists some $\eta_c>0$ such that $\P(\Ai=a|\Xi=x)\ge\eta_c$ for any $a\in\calA$ and $x\in\calX$.
    \end{compactenum}
\end{assumption}

Note that the \textit{boundedness} assumption is not essential and we only impose it for simplicity in our analysis. With additional effort, we can instead rely on light-tail distributional assumptions such as sub-Gaussian potential outcomes as in \citep{athey2021policy}. \textit{Unconfoundedness} ensures that action assignment is as good as random after accounting for measured covariates, and it is necessary to ensure valid policy value estimation
using inverse propensity-weighted strategies.
The \textit{uniform overlap} condition ensures that all actions are taken sufficiently many times to guarantee accurate evaluation of any policy.
This assumption may not be entirely necessary as recent work \citep{jin2022policy} introduced an approach that does away with the uniform overlap assumption for all actions and only relies on overlap for the optimal policy.
However, in our work, we made the above assumptions to simplify our analysis and maintain the focus of our contributions on the effects of data heterogeneity on policy learning.
In any case, these stated assumptions are standard and they are satisfied in many experimental settings such as randomized controlled trials or A/B tests.

Next, we also impose the following local data scaling assumption on each client. 
\begin{assumption}[Local Data Scaling]\label{ass:LocalDataSizeScaling}
    All local sample sizes asymptotically increase with the total sample size, i.e., for each $c\in\calC$, $n_c=\Omega(\nu_c(n))$ where $\nu_c$ is an increasing function of the total samples size $n$.
\end{assumption}
This assumption states that, asymptotically, the total sample size cannot increase without increasing across all data sources. We emphasize that this assumption is quite benign since $\nu_c$ could be any slowly increasing function (e.g., an iterated logarithm) and the asymptotic lower bound condition even allows step-wise increments. We only impose this assumption to ensure that the regret bounds in our analysis scale with respect to the total sample size with sensible constants. However, it does come at the cost of excluding scenarios in which a client always contributes $O(1)$ amount of data relative to the total data, no matter how much more total data is made available in aggregate, in which case one may expect it is better to exclude any such client.

\section{Approach}\label{sec:Approach}

The approach for the central server is to use the available observational data to construct an appropriate estimator of the global policy value and use this estimator to find an optimal global policy.

\subsection{Nuisance Parameters}
We define the following functions which we refer to as \textit{nuisance parameters} as they will be required to be separately known or estimated in the policy value estimates.

\begin{definition}
    The local \textit{conditional response} function $\mui$ and \textit{inverse conditional propensity} function $\oi$ of client $c\in\calC$
    with complete local data-generating distribution $(\Xi,\Ai,\Yi(a_1),\dots,\Yi(a_d))\sim\calDcec$
    are defined, respectively, for any $x\in\calX$ and $a\in\calA$, as
    \begin{gather*}
        \mui(x;a)\coloneqq\E[\Yi(a)|\Xi=x] \\
        \oi(x;a)\coloneqq1/\P(\Ai=a|\Xi=x).
    \end{gather*}
    For notational convenience, we let $\mui(x)=(\mui(x;a))_{a\in\calA}$ and $\oi(x)=(\oi(x;a))_{a\in\calA}$.
\end{definition}
In our estimation strategy, each client must estimate the conditional response and inverse conditional propensity functions when they are unknown. Following the literature on double machine learning \cite{chernozhukov2018double}, we make the following high-level assumption on the estimators of these local nuisance parameters.
\begin{assumption}\label{ass:FiniteSampleError}
    For any client $c\in\calC$, the local estimates $\muhati$ and $\ohati$ of the nuisance parameters $\mui$ and $\oi$, respectively, trained on $m$ local data points satisfy
    the following squared error bound:
    \begin{equation*}
        \resizebox{\linewidth}{!}{$\E\!\bigbra{\!\norm{\muhati(\Xi)\!-\!\mui(\Xi)}_2^2\!}\!\cdot\!\E\!\bigbra{\!\norm{\ohati(\Xi)\!-\!\oi(\Xi)}_2^2\!}\le\frac{o(1)}{m}$,}
    \end{equation*}
    where the expectation is taken with respect to the marginal distribution of $\calDc$ over contexts.
\end{assumption}

We emphasize this is a standard assumption in the double machine learning literature,
and we can easily construct estimators that satisfy these rate conditions, given sufficient regularity on the nuisance parameters \citep{zhou2023offline}.
See Appendix \ref{app:NuisanceParameterEstimation} for more details.
They can be estimated with widely available out-of-the-box regression and classification implementations.
Moreover, this condition is general and flexible enough to allow one to trade-off the accuracies of estimating the nuisance parameters. This is an important property in offline policy learning where distribution shift in the batch data can complicate consistent reward estimation.

\subsection{Policy Value Estimator}\label{sec:Approach-PolicyValueEstimators}

Next, we define our policy value estimators.
For any client $c\in\calC$,
we define the \textit{local augmented inverse propensity weighted} (AIPW) score
for each $a\in\calA$
to be
\begin{equation*}
    \Gi(a)\coloneqq\mui(\Xi;a)+\widebar{Y}^c(\Ai)\cdot\oi(\Xi;a)\cdot\ones\{\Ai=a\},
\end{equation*}
where
\begin{math}
\widebar{Y}^c=\Yi(\Ai)-\mui(\Xi;a)
\end{math}
are the centered outcomes
and
$(\Xi,\Ai,\Yi(a_1),\dots,\Yi(a_d))\sim\calDcec$.
One can readily show that $\Qi(\pi)=\opsE_{\calDcec}[\Gi(\pi(\Xi))]$, and therefore $\Qw(\pi)=\opsE_{\w}\opsE_{\calDcec}[\Gi(\pi(\Xi))]$ (see the proof in Lemma \ref{lem:ExpectedOracleEqualsLocalPolicyValue}). Accordingly, our procedure is to estimate the local AIPW scores and appropriately aggregate them to form the global policy value estimator.

We assume we have constructed nuisance parameter estimates $\muhati$ and $\ohati$ that satisfy Assumption \ref{ass:FiniteSampleError}. Then, for each data point $(\Xij,\Aij,\Yij)$ in the local observational data set of client $c\in\calC$, we define the \textit{approximate local AIPW} score for each $a\in\calA$ to be
\begin{equation*}
    \Ghatij(a)\coloneqq\muhati(\Xij;a)+\hat{\widebar{Y}}_i^c\cdot\ohati(\Xij;a)\cdot\ones\{\Aij=a\},
\end{equation*}
where
\begin{math}
    \hat{\widebar{Y}}_i^c=\Yij-\muhati(\Xij;a)
\end{math}
are the approximate centered outcomes.
Using these estimated scores, we introduce the following \textit{doubly robust global policy value estimate}:
\begin{equation*}
    \Qhatw(\pi)=\opsE_{c\sim\w}\bra{\frac{1}{\ni}\sum_{i=1}^{\ni}\Ghatij(\pi(\Xij))}.
\end{equation*}
Our proposed estimator is a generalized aggregate version of the doubly robust policy value estimator introduced in the standard offline policy learning setting \citep{zhou2023offline}.
It is doubly robust in the sense that it is accurate as long as one of nuisance parameter estimates is accurate for each client.
To ensure we can use the same data to estimate the nuisance parameters and to construct the policy value estimates, we utilize a \textit{cross-fitting} strategy locally for each client.
See Appendix \ref{app:CAIPWEstimation} for more details on this cross-fitting strategy.

\subsection{Optimization Objective}\label{sec:OptimizationObjective}

The objective of the central server is to find a policy that maximizes the doubly robust global policy value estimate:
\begin{equation*}
    \smash[t]{\pihatw=\argmax_{\pi\in\Pi}\Qhat_\w(\pi)}.
\end{equation*}
\vspace{-1.5em}

Note that in the centralized setting, this optimization can be done using standard policy optimization techniques \citep{athey2021policy} on the centrally accumulated heterogeneous datasets with appropriate reweighting.
However, as we discussed previously, the centralized collection of datasets can present difficulties in privacy sensitive settings.

For this reason, we seek to provide an optimization procedure that is also amenable to federated settings to overcome these challenges. 
In the federated setting, the central server does not have access to client raw data to estimate local policy values nor does it have access to the local policy values; only model updates can be shared through the network.
In Section \ref{sec:Algorithm}, we discuss our optimization procedure for parametric policy classes that manages these constraints.

\section{Regret Bounds}\label{sec:RegretBounds}

In this section, we establish regret bounds for the global policy solution $\pihatw$ to the optimization objective above. Refer to Appendices \ref{app:AuxiliaryResults}, \ref{app:PolicyClassComplexityMeasures}, \ref{app:BoundingGlobalRegret}, \ref{app:BoundingLocalRegret} for a detailed discussion and proofs of the results in this section.

\vspace{-.5em}
\subsection{Complexity and Skewness}\label{sec:ComplexityAndSkewness}

First, we introduce important quantities that appear in our regret bounds.

\vspace{-.75em}
\paragraph{Policy Class Complexity}
The following quantity provides a measure of policy class complexity based on a variation of the classical entropy integral introduced by \citet{dudley1967sizes}, and it is useful in establishing a class-dependent regret bound. See Appendix \ref{app:HammingDistance} for more details on its definition.

\begin{definition}[Entropy integral]
Let
\begin{math}
    \Ham(\pi_1,\pi_2;\tildex)\coloneqq\frac{1}{\tilden}\sum_{i=1}^{\tilden}\ones\{\pi_1(\tilde{x}_i)\neq\pi_2(\tilde{x}_i)\}
\end{math}
be the Hamming distance between any two policies $\pi_1,\pi_2\in\Pi$ given a covariate set $\tildex\subset\calX$ of size $\tilden\in\N$.
The \textit{entropy integral} of a policy class $\Pi$ is
\vspace{-.25em}
\begin{equation*}
    \kappa(\Pi)\coloneqq\int_0^1\sqrt{\log N_{\Ham}(\epsilon^2,\Pi)}d\epsilon,
\end{equation*}
\vspace{-1.25em}

where $N_{\Ham}(\epsilon^2,\Pi)$ is the maximal $\epsilon^2$-covering number of $\Pi$ under the Hamming distance over covariate sets of arbitrary size.
\end{definition}
The entropy integral is constant for a fixed policy class, and rather weak assumptions on the class are sufficient to ensure it is finite such as sub-exponential growth on its Hamming covering number, which is satisfied by many policy classes including parametric and finite-depth tree policy classes \citep{zhou2023offline}. In the binary action setting, the entropy integral of a policy class relates to its VC-dimension with $\kappa(\Pi)=\sqrt{\smash[b]{\VC(\Pi)}}$, and for $D$-dimensional linear classes $\kappa(\Pi)=\calO(\sqrt{\smash[b]{D}}\,)$.

\vspace{-.75em}
\paragraph{Client Skewness}
The following quantity measures the imbalance of the client sampling distribution $\w$ relative to the \textit{empirical distribution of samples
across clients} defined by $\barn\coloneqq(\ni/n)_{c\in\calC}$.
This quantity naturally arises in the generalization bounds of weighted mixture distributions \citep{mansour2021theory}.

\begin{definition}[Skewness]
    The \textit{skewness} of a given client sampling distribution $\w$ is
    \begin{equation*}
        \skewness\coloneqq\opsE_{c\sim\w}\bra{\frac{\wi}{\barni}}=1+\chi^2(\w||\barn),
    \end{equation*}
    where $\barn\coloneqq(\ni/n)_{c\in\calC}$ and $\chi^2(\w||\barn)$ is the chi-squared divergence of $\w$ from $\barn$.
\end{definition}

\subsection{Global Regret Bound}

The following result captures a root-$n$ finite-sample bound for the global regret that parallels the optimal regret bounds typically seen in the offline policy learning literature.
In Appendix \ref{app:ProofSketch}, we provide a high-level sketch of the proof.

\begin{restatable}[Global Regret Bound]{theorem}{MainTheorem}\label{thm:MainTheorem}
    Suppose Assumption \ref{ass:dgp}, \ref{ass:LocalDataSizeScaling}, and \ref{ass:FiniteSampleError} hold.
    Then, with probability at least $1-\delta$,
    \begin{equation*}
        \Rw(\pihatw)\le C_{\Pi,\delta}\cdot\sqrt{V\cdot\frac{\skewness}{n}} + \littleop{\!\sqrtfracskewnessn},
    \end{equation*}
    where
    \begin{gather*}
        V=\max\limits_{c\in\calC}\sup\limits_{\pi\in\Pi}\E_{\calDcec}\!\bigbra{\Gi(\pi(\Xi))^2}, \\
        C_{\Pi,\delta}=c_1\kappa(\Pi)+\sqrt{c_2\log(c_2/\delta)},
    \end{gather*}
    and $c_1,c_2$ are universal constants.
\end{restatable}

First, note that $V$ captures a notion of the worst-case AIPW score variance across clients.
Next, we observe that root-$n$ rate is moderated by a skewness term which can also scale with the total sample size.
For example, if $\w=\barn$ then $\skewness/n=1/n$, and if $\w=(1,0,\dots,0)$ then $\skewness/n=1/n_1$. Thus, this skewness-moderated rate generalizes and smoothly interpolates between the rates one expects from the uniform weighted model and the single source model. Indeed, when clients are identical and $\w=\barn$, we recover the best known rates from standard offline policy learning \citep{zhou2023offline}.

From this observation, it may seem that the best design choice for the client sampling distribution is the empirical sample distribution, i.e., $\w=\barn$. However, as we will observe in the next section, there are terms in the local regret bounds that introduce trade-offs on the choice of $\w$ when considering a specific target client.

\subsection{Local Regret Bound}
In this next result, we capture the discrepancy in local and global regret due to client heterogeneity.
This result is helpful in understanding the extent at which the global and local regret minimization objectives can be in conflict for a particular target client.

\begin{restatable}[Local Regret Bound]{theorem}{SubMainTheorem}\label{thm:SubMainTheorem}
    Suppose Assumption \ref{ass:dgp} holds. Then, for any client $c\in\calC$,
    \begin{equation*}
        \Ri(\pihatw)\le U\cdot \TV(\calDcec,\calDwew) + \Rw(\pihatw),
    \end{equation*}
    where $\TV$ is the total variation distance and $U=3B/\eta$ with $B=\maxC B_c$ and $\eta=\minC\eta_c$.
\end{restatable}

Note that the first term in this regret bound is inherently irreducible relative to the sample sizes and it is due to distribution shift between the complete local client distribution $\calDcec$ and the complete global mixture distribution $\calDwew$.
Thus, we can observe how the design choice on the client distribution $\w$ must balance a trade-off to achieve low skewness and low expected distribution shift across sources. Skewing towards the target client will reduce the distribution shift term, but it will further eat at the rates in the global regret bound.
In our experiments, we take a heuristic approach to obtain a skewed $\w$, but we note that this may be chosen in a more principled manner, e.g., minimax skewness as in \citep{mohri2019agnostic}.
Observe that the constant $U$ in the distribution shift term is defined by the constants in the boundedness and overlap assumptions stated in Assumption \ref{ass:dgp}.
However, in Appendix \ref{app:AlternativeLocalRegretBound}, we provide an alternate upper bound that does not rely on bounded AIPW scores and instead is scaled by the worst-case AIPW variance, which can be smaller and also appears in our global regret bound.

Next, we demonstrate how the distribution shift term can be further tensorized into contributions due to distribution shift in the covariates, propensities, and potential outcomes.

\begin{restatable}[Local Distribution Shift Bound]{theorem}{SubSubMainTheorem}\label{thm:SubSubMainTheorem}
    For any given client $c\in\calC$, suppose $\smash[tb]{(\Xi,\vecYi)\sim\calDc}$. We let $\smash[tb]{p_{\Xi}}$ denote the marginal distribution of $\Xi$ and let $\smash[tb]{p_{\vecYi|\Xi}}$ denote the conditional distribution of $\smash[tb]{\vecYi}$ given $\Xi$. Recall $e_c$ denotes the local conditional propensity function.
    Then, the irreducible distribution shift term in the local regret bound can be further bounded as
    \begin{equation*}
    \begin{split}
        \TV(\calDcec,\calDwew)\le\opsE_{k\sim\w}\Bigbra{&\sqrt{\smash[b]{\KL(p_{\Xi}||p_{\Xk})}}+\sqrt{\smash[b]{\KL(\ei||\ek)}} \\
        &+\sqrt{\smash[b]{\KL(p_{\vec{Y}^c|\Xi}||p_{\vec{Y}^k|\Xk})}}\,},
    \end{split}
    \end{equation*}
    where $\TV$ is the total variation distance and $\KL$ is the Kullback-Leibler divergence.\footnote{Note that the last two terms in the expectation of this inequality are conditional KL divergences on $p_{\Xi}$. See Appendix \ref{app:DistributionShiftBound} for more details.}
\end{restatable}

This result directly reveals the contribution to local regret from each possible source of distribution shift. This implies that if we have prior knowledge that certain components of the data-generating distribution match, then we can claim tighter bounds on the local regret of clients.

In summary, the analysis we presented in this section help provide insights into
understanding the tradeoffs and value of information in heterogeneous client participation in offline policy learning
(see Appendix \ref{app:ValueOfInformation} for further discussion).

\section{Algorithm}\label{sec:Algorithm}

We present FedOPL, a federated algorithm for finding the optimal global policy $\pihatw$ for the optimization problem stated in Section \ref{sec:OptimizationObjective}.
The standard approaches in federated learning are based on the federated averaging (FedAvg) algorithm \citep{konevcny2016federated,mcmahan2017communication}, which works iteratively by selecting clients to participate in the training round, fine-tuning a local parametric model instance on each client using their own data, and then aggregating local model parameters on the central server via a weighted average.
To enable standard federated learning strategies for policy optimization, we consider parametric policy classes $\Pi_\Theta=\{\pi_\theta:\calX\to\calA\mid\theta\in\Theta\}$, and we construct an iterative parametric policy optimization procedure for the local policy updates.

Specifically, we observe that the corresponding local policy optimization procedures
\begin{math}
    \argmax_{\theta\in\Theta}\Qhati(\pi_\theta)
\end{math}
are equivalent to cost-sensitive multi-class classification (CSMC) problems \citep{beygelzimer2009error, dudik2011efficient}, where the actions are the labels and the AIPW scores are the negative costs for the labels. Therefore, we can conduct iterative local policy model updates using online CSMC oracle methods \citep{beygelzimer2008machine, agarwal2017effective}, often used for policy learning in contextual bandit algorithms \citep{agarwal2014taming, bietti2021contextual}.

See Algorithms \ref{alg:FedAvgCSMC_Server} and \ref{alg:FedAvgCSMC_Client} for the server-side and client-side implementations of Fed-OPL based on online CSMC local model updates.
Note that prior to policy learning, each client must estimate their local AIPW scores, as described in Section \ref{sec:Approach-PolicyValueEstimators}. For data efficiency in constructing these scores, we use a cross-fitting strategy on local observational data to estimate local nuisance parameters. See Algorithm \ref{alg:CAIPW} in Appendix \ref{app:CAIPWEstimation} for more details on the implementation of this AIPW estimation procedure.

\vspace{-.75em}
\paragraph{Remark on Optimality Gap}
Note that obtaining the optimal global policy is necessary for achieving the stated regret bounds above.
FedOPL is guaranteed to converge to the optimal policy if the optimization problem is concave.
However, this only holds for certain policy classes such as linear policy classes.
Nevertheless, we can still establish useful regret bounds for nearly optimal policies under general policy classes. The regret bound is simply modified to include an additive term that captures the policy value optimality gap. In particular, if $\tilde\pi_\w$ is the approximate global policy, then we can easily extend the global regret bound to be
\begin{math}
    \Rw(\tilde\pi_\w)\le \Delta_\w(\pihatw,\tilde\pi_\w) + O_p(\sqrt{\skewness/n}),
\end{math}
where the asymptotic term is the exact bound established above and the first term $\Delta_\w(\pihatw,\tilde\pi_\w)=\Qw(\pihatw)-\Qw(\tilde\pi_\w)$ is the policy value gap between the optimal global policy $\pihatw$ and the approximate global policy $\tilde\pi_\w$.
Under a judicious choice of policy class and corresponding optimization procedure, the value gap can be rendered insignificant or of a similar order of magnitude as other terms in the regret bounds,
especially under
heterogeneous environments where there is significant irreducible distribution shift.

\begin{figure}
    \vspace{-1.1em}
        \begin{algorithm}[H]
        \caption{FedOPL: Server-Side}
        \label{alg:FedAvgCSMC_Server}
        \begin{algorithmic}[1]
        \REQUIRE clients $\calC$, client distribution $\w$
        \vspace{2pt}
        \STATE Initialize global model parameters $\theta_g$
        \FOR{each round $t=1,2,\dots$}
            \STATE Sample a subset of clients $\calS\subset\calC$
            \FOR{each client $c\in\calS$ \textbf{in parallel}}
                \STATE Send global parameters $\theta_g$ to client $c$
                \STATE Await local updates $\theta_c$ from client $c$
            \ENDFOR
            \STATE Update global parameters: \\ 
            \hspace{1em} $\theta_g \leftarrow \sum_{c\in\calS}\wi \theta_c/\sum_{c\in\calS}\wi$
        \ENDFOR
        \end{algorithmic}
        \end{algorithm}
        \vspace{-2em}
        \begin{algorithm}[H]
        \caption{FedOPL: Client-Side}
        \label{alg:FedAvgCSMC_Client}
        \begin{algorithmic}[1]
        \REQUIRE local steps $T$, local batch size $B$,  \\
        \vspace{1pt}
        local data $\{(\Xij,\Ghatij(a_1),\dots,\Ghatij(a_d))\}_{i=1}^{\nc}$
        \vspace{3pt}
        \STATE Receive global parameters $\theta_g$ from server
        \STATE Initialize local parameters $\theta_c\leftarrow\theta_g$
        \vspace{1pt}
        \FOR{$t=1,\dots,T$}
            \STATE $\calB \leftarrow$ sample a batch of $B$ local examples
            \STATE Update local parameters using online CSMC oracle: \\
            \hspace{1em}$ \theta_c \leftarrow \text{CSMC}(\theta_c, \calB)$
        \ENDFOR
        \vspace{1pt}
        \STATE Send local parameters $\theta_c$ to server
        \end{algorithmic}
        \end{algorithm}
    \vspace{-2.5em}
\end{figure}

\vspace{-.15em}
\section{Experiments}\label{sec:Experiments}

For our experiments, we compare empirical local and global regrets across different experimental settings involving homogeneous and heterogeneous clients. 

\subsection{Setup}
We describe the experimental setup common to all of our experiments.
For the environment, we consider the client set $\calC=[C]$ with $C=3$, the action set $\calA=\{a_1,\dots,a_d\}$ with $d=4$, and the context space $\calX=\R^p$ with $p=d\times q$ where $q=10$.
For every client $c\in\calC$, we consider the following data-generating process:
\vspace{-.35em}
\begin{compactitem}
    \item $\Ai\sim\Uniform(\calA)$,
    \item $\Xi\sim\Normal(0,\sigma_c^2 I_p)$,
    \item $\Yi(a)|\Xi\sim\Normal(\mui(\Xi;a),\rho_c^2)$ for all $a\in\calA$,
\end{compactitem}
\vspace{-.35em}
where the noise parameters $\sigma_c, \rho_c$ and the true conditional reward functions $\mui$ are specified in each experiment below.
For a given total sample size $n\in\N$, each client $c\in\calC$ is allocated a local sample sample size determined by an increasing function $\nc=\nuc(n)$, which we specify in each experiment below.
By construction, this entire data-generating process satisfies our data assumptions stated in Section \ref{sec:Preliminaries-Data}.

For the parametric policy class, we consider the policy class induced by linear scores: $\pi_\theta(x)=\argmax_{a\in\calA}\phi(x,a)^\top \theta_a$ where $\theta\in\Theta=\R^{d\times q}$ and $\phi:\calX\times\calA\to\R^q$ is a feature extractor, which we simply set to $\phi(x,a)=x_a$.
Then, for the local CSMC oracle, we employ the cost-sensitive one-against-all (CSOAA) implementation in Vowpal Wabbit \citep{vowpalwabbit}. This method performs online multiple regressions of costs (i.e., negative rewards) on contexts for each action, and at inference time, it selects the action with the lowest predicted cost \citep{agarwal2017effective}.

For our results, we consider a training sample size grid in the range up to $N$ samples, where $N=1$K for the homogeneous experiments and $N=10$K for the heteregeneous experiments (longer range due to slower convergence). For each sample size $n$ in this grid, we sample $\nc=\nuc(n)$ training samples from each client distribution defined above, and we train a global policy $\pihatw$ to optimize the corresponding global policy value estimate.
We will then compute and plot the empirical estimates of the global regret $\Rw(\pihatw)$ and local regret $\Ri(\pihatw)$ for a fixed client $c$.
Our plots will display the mean and the standard deviation bands of these empirical regrets over five different seed runs.
We compute these metrics on 10K separate test samples from each client. The empirical regrets are computed against the optimal models trained on 100K total separate training samples.

\subsection{Baselines}

For comparison against pure local training, we will also plot the local regret $\Ri(\pihat_c)$ of the locally trained policy $\pihat_c$. The locally trained policy will be trained with the same number $n$ of total sample sizes across all clients.
This will help us understand the benefits of federation compared to simply training a policy locally.

We will also compare against the HierOPO algorithm for multi-task offline policy learning introduced by \citet{hong2023multi}.
HierOPO is a hierarchical Bayesian approach where task similarity is captured using an unknown latent parameter.
Although this approach can be applied to our setting, we note significant limitations.
This approach assumes the reward model is known and, without access to an informative prior on the latent task parameter, the regret bounds only scale with the local sample size.
Furthermore, this algorithm is less well-suited for federation since it relies on aggregation of summary statistics across tasks rather than model parameters, thereby increasing communication complexity for higher-dimensional problems and increasing privacy risks.
In our experiments, we will let $\pihat_H$ denote the policy trained under the HierOPO algorithm.

\subsection{Homogeneous Clients}

For sanity check, we consider the homogeneous setting where all clients are identical. We should expect the local and global regrets to be the same.
For every client $c\in\calC$, we set
$\mui(x;a)=\phi(x,a)^\top\theta_a$
where $\theta\sim\Normal(0,\omega^2 I_p)$ with $\omega^2=1$ and $\sigma_c^2=\rho_c^2=1$.
Moreover, we simply divide the total sample size across all clients, i.e., $\nu_c(n)=\lfloor n/C\rfloor$.
Since all clients are the same, we simply use the empirical mixture $\w=\barn$ in our optimization procedure.
Figure \ref{fig:hom} displays our results.
As expected, each of these curves is nearly identical since the global and local regrets are identical in this scenario.

\begin{figure}[ht]
    \centering
    \begin{minipage}{.84\linewidth}
        \centering
        \includegraphics[trim=5pt 0pt 50pt 23pt, clip, width=\linewidth]{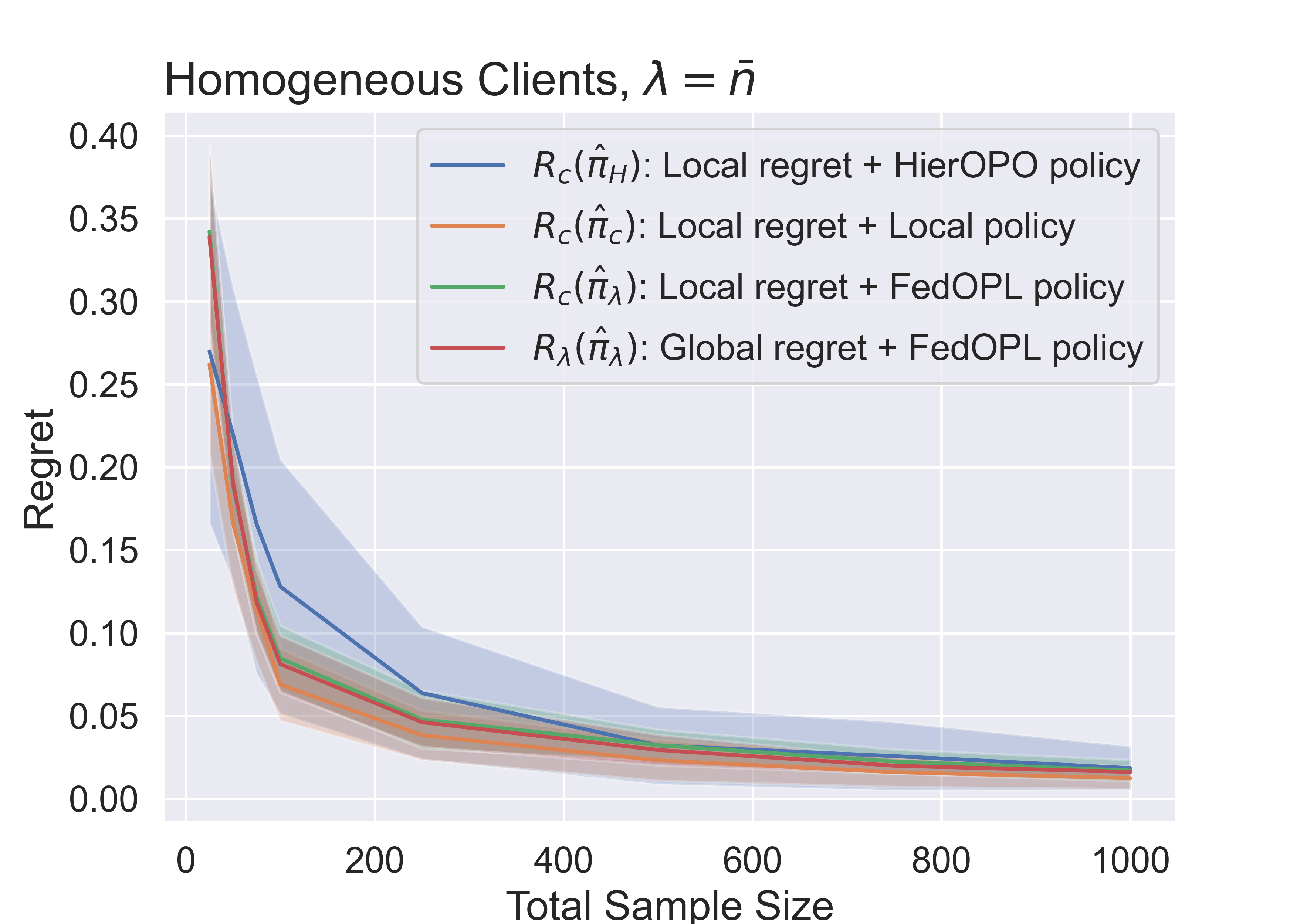}
    \end{minipage}
    \vspace{-.5em}
    \caption{Empirical regret curves under homogeneous clients. All local regrets shown are shown for client $c=1$.}
    \label{fig:hom}
    \vspace{-1.25em}
\end{figure}

\subsection{Heterogeneous Clients}
\vspace{-.25em}

Next, we consider the more interesting setting where one client is different than all other clients. Specifically, we set
\begin{compactitem}
    \item $\mu_c(x,a)\!=\!\phi(x,a)^\top\theta_a$ and $\sigma_c^2=\rho_c^2=5$ for $c\neq 1$,
    \item $\mu_1(x;a)\!=\!k\sin(\phi(x,a)^\top\theta_a/k)$ and $\sigma_1^2=\rho_1^2=10$,
\end{compactitem}
where $k=50$ and $\theta\sim\Normal(0,\omega^2I_p)$ with $\omega^2=5$.
The above hyperparameters were simply chosen to better highlight the differences between regret curves, but we found our results to be robust to different hyperparameter choices.
Under this data-generating process, beyond the distribution shift in the contexts and rewards due to different variances, there is also distribution shift due to different mean rewards.
The idea behind this choice of reward function is that since the scaled sine function is nearly linear near zero, there is a wide range of contexts where the reward function for client $c=1$ is nearly identical to the linear reward function of all other clients.
Therefore, there is distribution shift between client $c=1$ and all other clients, but there is some amount of similarity that can be exploited.
Moreover, to illustrate the benefits of federation under sample size heterogeneity, we will also have client $c=1$ contribute significantly less data than the others with $n_1=\nu_1(n)=\lfloor\log n\rfloor$ and all other clients will evenly distribute the rest of the total sample size.

\vspace{-.75em}
\paragraph{Empirical Client Sampling Distribution}
The top figure in Figure \ref{fig:side_by_side} plots the same type of regret curves as in the homogeneous experiment. In FedOPL, we use the empirical client sampling distribution $\w=\barn$. We observe that
$R_1(\pihatw)$
significantly suffers from distribution shift, as predicted. In fact, the locally trained policy
$\pihat_1$
performs better than the globally trained policy
$\pihatw$
at a sufficiently large sample size.
Note that the HierOPO policy only outperforms the locally trained model at low sample sizes due to worse scaling.

\vspace{-.75em}
\paragraph{Skewed Client Sampling Distribution}
The bottom figure in Figure \ref{fig:side_by_side} plots similar regret curves, but instead with the global policy trained with a skewed client sampling distribution $\w=\barn+\bar{\varepsilon}$ where $\bar{\varepsilon}_c=-\alpha\cdot\barn_c$ for $c\neq 1$ and $\bar{\varepsilon}_1=\alpha\cdot(1-\barn_1)$ for $\alpha=0.2$. 
Here, we observe that
$\pihatw$
still suffers some amount in terms of local regret, but not to such an extent that
$\pihat_1$
beats it. Moreover, the local regret shift decreases with larger sample size. The idea is that this skewness upscales the distribution of client 1 to diminish the amount of distribution shift of $\bar{\calD}_1$ from
$\calDwew$
especially at larger total sample sizes, at the cost of negatively affecting the other more homogeneous clients.
Refer to the results in Appendix \ref{app:AdditionalExperimentalResults} to see how the more similar clients are affected by this design choice modification to favor client 1. The takeaway in those experiments is that the other clients have less distribution shift from the average so their performance degradation is lesser under $\w=\barn$, but their performance further slightly degrades under the skewed client sampling distribution $\w=\barn+\bar{\varepsilon}$.

\begin{figure}[ht]
    \centering
    \begin{minipage}{.825\linewidth}
        \centering\includegraphics[trim=8pt 0pt 50pt 23pt, clip, width=\linewidth]{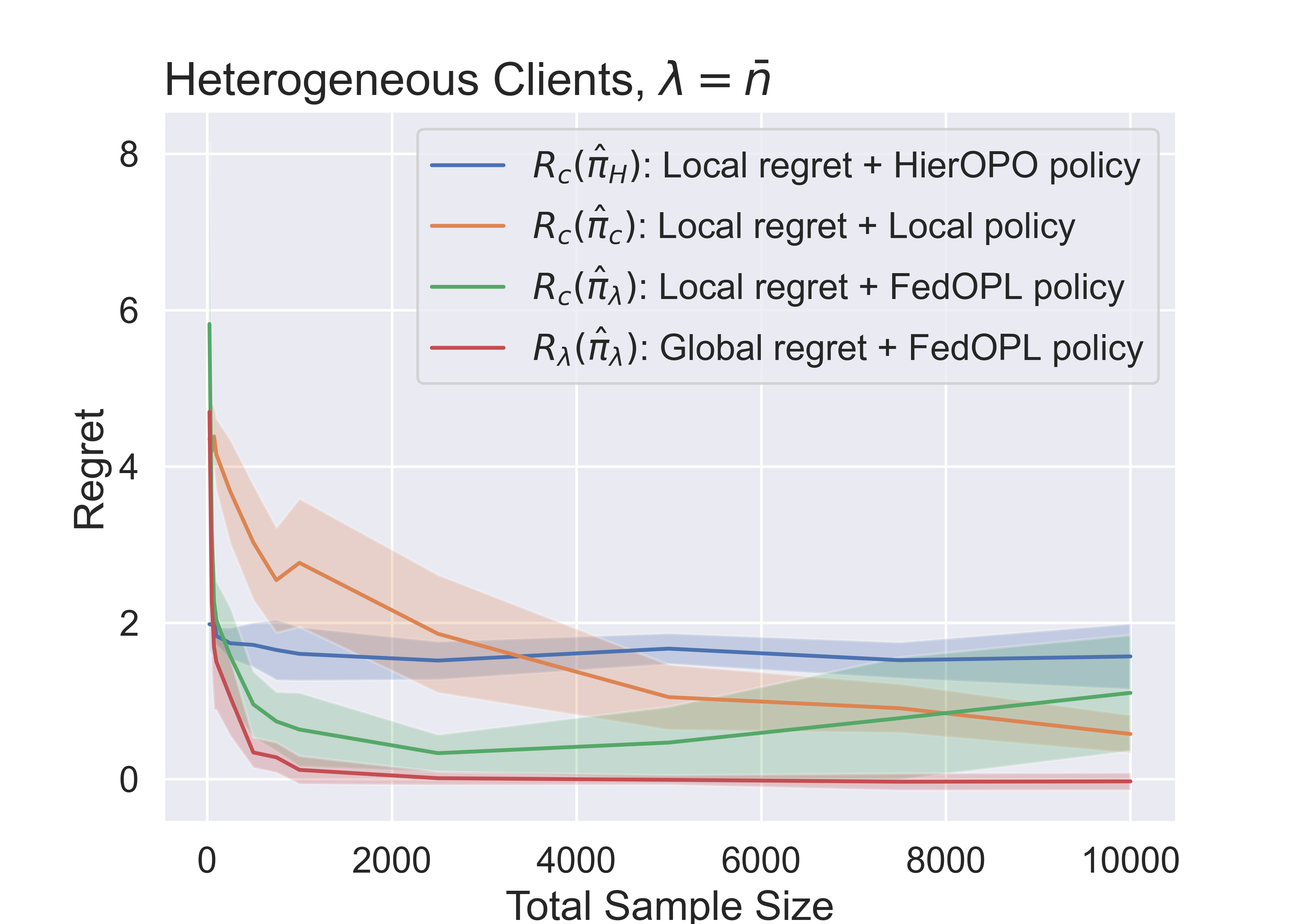}
        \label{fig:image2}
        \vspace{-.5em}
    \end{minipage}
    \begin{minipage}{.825\linewidth}
        \centering
        \includegraphics[trim=8pt 0pt 50pt 23pt, clip, width=\linewidth]{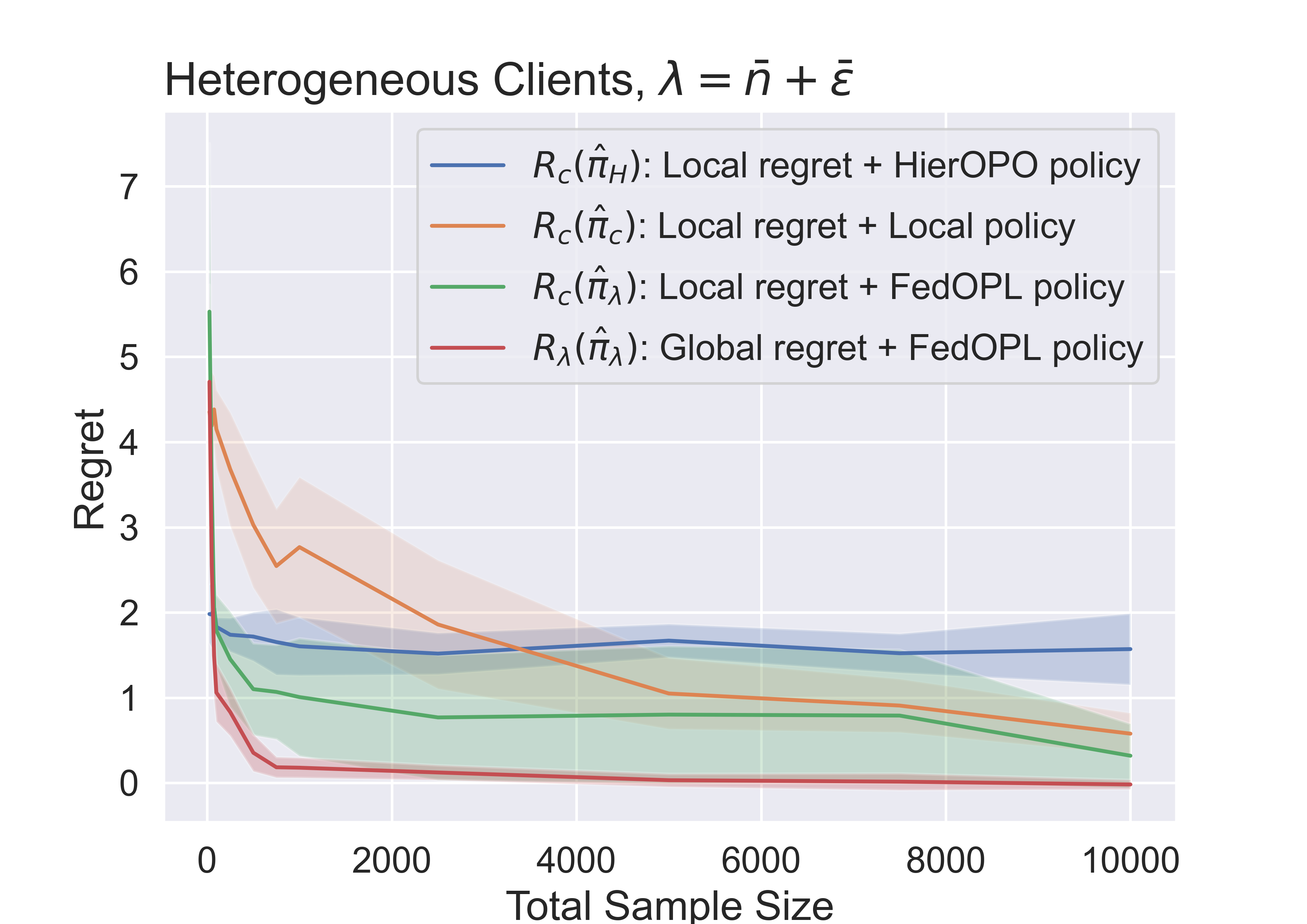}
        \label{fig:image3}
    \end{minipage}
    \vspace{-1.5em}
    \caption{Empirical regret curves under heterogeneous clients. Top: $\lambda=\barn$; Bottom: $\lambda=\barn+\bar{\varepsilon}$. All local regrets shown are shown client $c=1$.}
    \label{fig:side_by_side}
    \vspace{-1.5em}
\end{figure}

\section{Conclusion}
We studied the problem of offline policy learning from observational bandit feedback data across multiple heterogeneous data sources.
Moreover, we considered the practical aspects of this problem in a federated setting to address privacy concerns.
We presented a novel regret analysis and supporting experimental results that demonstrate tradeoffs of client heterogeneity on policy performance.
In Appendix \ref{app:AdditionalDiscussion}, we provide additional discussion on the broader impact, limitations, and potential future directions to our work.

\bibliography{references}
\bibliographystyle{icml2024}

\newpage
\appendix
\onecolumn
\section{Auxiliary Results}\label{app:AuxiliaryResults}

The following known results will be used in our regret bound proofs. See Chapter 2 of \cite{koltchinskii2011oracle} for discussions of these results.

\begin{lemma}[Hoeffding's inequality]\label{lem:HoeffdingInequality}
    Let $Z_1,\dots,Z_n$ be independent random variables with $Z_i\in[a_i,b_i]$ almost surely. For all $t>0$, the following inequality holds
    \begin{equation*}
        \P\rbra{\biggabs{\sum_{i=1}^nZ_i-\E\bra{Z_i}}\ge t}\le 2\exp\rbra{-\frac{2t^2}{\sum_{i=1}^n(b_i-a_i)^2}}.
    \end{equation*}
\end{lemma}

\begin{lemma}[Talagrand's inequality]\label{lem:TalagrandInequality}
    Let $Z_1,\dots,Z_n$ be independent random variables in $\calZ$. For any class of real-valued functions $\calH$ on $\calZ$ that is uniformly bounded by a constant $U>0$ and for all $t>0$, the following inequality holds
    \begin{equation*}
        \P\rbra{\biggabs{\supH\Bigabs{\sum_{i=1}^n h(Z_i)} - \E\Bigbra{\supH\Bigabs{\sum_{i=1}^n h(Z_i)}}}\ge t}\le C\exp\rbra{-\frac{t}{CU}\log\rbra{1+\frac{Ut}{D}}},
    \end{equation*}
    where $C$ is a universal constant and $D\ge\E\bra{\supH\sum_{i=1}^n h^2(Z_i)}$.
\end{lemma}

\begin{lemma}[Ledoux-Talagrand contraction inequality]\label{lem:ContractionInequality}
    Let $Z_1,\dots,Z_n$ be independent random variables in $\calZ$. For any class of real-valued functions $\calH$ on $\calZ$ and any $L$-Lipschitz function $\varphi$, the following inequality holds
    \begin{equation*}
        \E\bra{\supH\biggabs{\sum_{i=1}^n\eps_i(\varphi\circ h)(Z_i)}}\le 2L\E\bra{\supH\biggabs{\sum_{i=1}^n\eps_i h(Z_i)}},
    \end{equation*}
    where $\eps_1,\dots,\eps_n$ are independent Rademacher random variables.
\end{lemma}

Lastly, we state an auxiliary inequality that serves as a typical candidate for the quantity denoted by $D$ above in Talagrand's inequality. This result follows as a corollary of the Ledoux-Talagrand contraction inequality and a symmetrization argument. We provide a proof for completeness.

\begin{lemma}\label{lem:ExpectedSupSumBound}
    Let $Z_1,\dots,Z_n$ be independent random variables in $\calZ$. For any class of real-valued functions $\calH$ on $\calZ$ and any $L$-Lipschitz function $\varphi$, the following inequality holds
    \begin{equation*}
        \E\bra{\supH\sum_{i=1}^n (\varphi\circ h)(Z_i)}\le\supH\sum_{i=1}^n\E\bigbra{(\varphi\circ h)(Z_i)}+4L\E\bra{\supH\biggabs{\sum_{i=1}^n\eps_ih(Z_i)}},
    \end{equation*}
    where $\eps_1,\dots,\eps_n$ are independent Rademacher random variables.
\end{lemma}
\begin{proof}
    We have that
    \begin{align}
        &\E\bra{\supH\sumin(\varphi\circ h)(Z_i)}-\supH\sumin\E\bigbra{(\varphi\circ h)(Z_i)} \label{eq:Lemma-ExpectedSupSumBound-eq1} \\
        &=\E\bra{\supH\sumin(\varphi\circ h)(Z_i)-\supH\sumin\E\bigbra{(\varphi\circ h)(Z_i)}} \\
        &\le\E\bra{\supH\biggabs{\sumin(\varphi\circ h)(Z_i)-\sumin\E\bigbra{(\varphi\circ h)(Z_i)}}} \label{eq:Lemma-ExpectedSupSumBound-eq3} \\
        &=\E\bra{\supH\biggabs{\sumin\Bigpar{(\varphi\circ h)(Z_i)-\E\bigbra{(\varphi\circ h)(Z_i)}}}} \\
        &\le2\E\bra{\supH\biggabs{\sumin\eps_i(\varphi\circ h)(Z_i)}} \label{eq:Lemma-ExpectedSupSumBound-eq5} \\
        &\le 4L\E\bra{\supH\biggabs{\sumin\eps_i h(Z_i)}}. \label{eq:Lemma-ExpectedSupSumBound-eq6}
    \end{align}
    Inequality \eqref{eq:Lemma-ExpectedSupSumBound-eq3} follows from the triangle inequality, inequality \eqref{eq:Lemma-ExpectedSupSumBound-eq5} follows from a standard symmetrization argument (see \cite{koltchinskii2011oracle}), and inequality \eqref{eq:Lemma-ExpectedSupSumBound-eq6} follows from the Ledoux-Talagrand contraction inequality (see Lemma \ref{lem:ContractionInequality}). The result follows by moving the second term in Equation \eqref{eq:Lemma-ExpectedSupSumBound-eq1} to the right-hand side in the last inequality.
\end{proof}

\section{Complexity and Heterogeneity Measures}\label{app:PolicyClassComplexityMeasures}

In this section, we introduce important quantities of policy class complexity and client heterogeneity that appear in our analysis. All throughout, we let $n=\sumiM\ni$ be the total sample size across clients, $\nC=(\nc)_{c\in\calC}$ the vector of sample sizes across clients, and $\barn=(\ni/n)_{c\in\calC}$ the empirical distribution over clients. 

\subsection{Policy Class Complexity}

\subsubsection{Hamming Distance \& Entropy Integral}\label{app:HammingDistance}

We provide additional details on the definition of the entropy integral introduces in Section \ref{sec:ComplexityAndSkewness}.

\begin{definition}[Hamming distance, covering number, and entropy integral]\label{def:HammingDistance-CoveringNumber-EntropyIntegral}
    Consider a policy class $\Pi$ and a multi-source covariate set $x=\x\subset\calX$ across clients $\calC$ with client sample sizes $\nC$. We define the following:
    \begin{enumerate}[(a)]
        \item the Hamming distance between any two policies $\pi_1,\pi_2\in\Pi$ given multi-source covariate set $x$  is
        \begin{equation*}
            \Ham(\pi_1,\pi_2;x)\coloneqq\frac{1}{\sumiM\nc}\sumiM\sumjni\ones\{\pi_1(\xij)\neq\pi_2(\xij)\};
        \end{equation*}
        
        \item an $\epsilon$-cover of $\Pi$ under the Hamming distance given covariate set $x$ is any policy set $S$ such that for any $\pi\in\Pi$ there exists some $\pi'\in S$ such that $\Ham(\pi,\pi';x)\le\epsilon$;
    
        \item the $\epsilon$-covering number of $\Pi$ under the Hamming distance given covariate set $x$ is
        \begin{equation*}
            N_\Ham(\epsilon,\Pi;x)\coloneqq\min\{|S|\mid S\in\calS_{\Ham}(\epsilon,\Pi;x)\},
        \end{equation*}
        where $\calS_{\Ham}(\epsilon,\Pi;x)$ is the set of all $\epsilon$-covers of $\Pi$ with respect to $ \Ham(\cdot,\cdot;x)$;
        
        \item the $\epsilon$-covering number of $\Pi$ under the Hamming distance is
        \begin{equation*}
            N_\Ham(\epsilon,\Pi)\coloneqq\sup\{N_\Ham(\epsilon,\Pi;x)\mid x\in\calX_\calC\},
        \end{equation*}
        where $\calX_\calC$ is the set of all covariate sets in $\calX$ across clients $\calC$ with arbitrary sample sizes;
        
        \item the entropy integral of $\Pi$ is
        \begin{equation*}
            \kappa(\Pi)\coloneqq\int_0^1\sqrt{\log N_\Ham(\epsilon^2,\Pi)}d\epsilon.
        \end{equation*}
    \end{enumerate}
\end{definition}

\subsubsection{\texorpdfstring{$\ellw$}{l2} Distance}\label{app:ellw2Distance}
Consider the function class
\begin{equation*}
    \calF_\Pi\coloneqq\{Q(\cdot,\pi):\Omega\to\R\mid\pi\in\Pi\},
\end{equation*}
where
\begin{equation*}
    Q(\omegaij;\pi)\coloneqq\gij(\pi(\xij))
\end{equation*}
for any covariate-score vector $\omegaij=(\xij,\gij)\in\Omega=\calX\times\R^d$ and $\pi\in\Pi$, where $\gij(a)$ is the $a$-th coordinate of the score vector $\gij$.

\begin{definition}[$\ellw$ distance and covering number]
    Consider a policy class $\Pi$, function class $\calF_\Pi$, and a multi-source covariate-score set $\omega=\{\omegaij\mid c\in\calC,i\in[\nc]\}\subset\Omega$ across clients $\calC$ with client sample sizes $\nC$ and client distribution $\w$. We define:
    \begin{enumerate}[(a)]
        \item the $\ellw$ distance with respect to function class $\calF_\Pi$ between any two policies $\pi_1,\pi_2\in\Pi$ given covariate-score set $\omega$ is
        \begin{equation*}
            \ellw(\pi_1,\pi_2;\omega)=\sqrt{\frac{\sumijwsq\bigpar{Q(\omegaij;\pi_1)-Q(\omegaij;\pi_2)}^2}{\sup_{\pi_a,\pi_b\in\Pi}\sumijwsq\bigpar{Q(\omegaij;\pi_a)-Q(\omegaij;\pi_b)}^2}};
        \end{equation*}
        
        \item an $\epsilon$-cover of $\Pi$ under the $\ellw$ distance given covariate-score set $\omega$ is any policy set $S$ such that for any $\pi\in\Pi$ there exists some $\pi'\in S$ such that $\ellw(\pi,\pi';\omega)\le\epsilon$;
        
        \item the $\epsilon$-covering number of $\Pi$ under the $\ellw$ distance given covariate-score set $\omega$ is
        \begin{equation*}
            N_{\ellw}(\epsilon,\Pi;\omega)\coloneqq\min\{|S|\mid S\text{ is an $\epsilon$-cover of $\Pi$ w.r.t.~$\ellw(\cdot,\cdot;\omega)$}\}.
        \end{equation*}
    \end{enumerate}
\end{definition}

\vspace{1em}
The following lemma relates the covering numbers of the two policy distances we have defined.

\begin{lemma}
Let $\omega=\{\omegaij\mid c\in\calC,i\in[\nc]\}\subset\Omega$ be a multi-source covariate-score set across clients $\calC$ with client sample sizes $\nC$ and client distribution $\w$. For any $\epsilon>0$,
\begin{equation*}
    N_{\ellw}(\epsilon,\Pi;\omega)\le N_{\Ham}(\epsilon^2,\Pi).
\end{equation*}
\end{lemma}
\begin{proof}
    Fix $\epsilon>0$. 
    Without loss of generality, we assume $N_{\Ham}(\epsilon^2,\Pi)<\infty$, otherwise the result trivially holds. Let $S_0=\{\pi_1,\dots,\pi_{N_0}\}$ be a corresponding Hamming $\epsilon^2$-cover of $\Pi$.

    Consider any arbitrary $\pi\in\Pi$. By definition, there exists a $\pi'\in S_0$ such that for any multi-source covariate set $\tildex=\{\tildexij\mid c\in\calC, i\in[\tilden_c]\}$ with any given sample sizes $\tilden_c>0$ the following holds:
    \begin{equation*}
        \Ham(\pi,\pi';\tildex)=\frac{1}{\tilden}\sumiM\sum_{i=1}^{\tilden_c}\ones\{\pi(\tildexij)\neq\pi'(\tildexij)\}\le\epsilon^2,
    \end{equation*}
    where $\tilden=\sumiM\tilden_c$. Using this pair of policies $\pi, \pi'$ we generate an augmented data set $\tilde\omega$ from $\omega$ as follows.
    Let $m$ be a positive integer and define $\tilde\omega$ to be a collection of multiple copies of all covariate-score tuples $\omegaij\in \omega$, where each $\omegaij$ appears
    \begin{equation*}
        \tildenij\coloneqq\left\lceil\frac{m\cdot\fracwinisq\bigpar{Q(\omegaij;\pi)-Q(\omegaij;\pi')}^2}{\sup_{\pi_a,\pi_b}\sumijwsq\bigpar{Q(\omegaij;\pi_a)-Q(\omegaij;\pi_b)}^2}\right\rceil    
    \end{equation*}
    times in $\tilde\omega$. Therefore, the client sample sizes in this augmented data set are $\tilden_c=\sumjni\tildenij$ and the total sample size is $\tilden=\sumij\tildenij$. The total sample size is bounded as
    \begin{align*}
        \tilden&=\sumij\left\lceil\frac{m\cdot\fracwinisq\bigpar{Q(\omegaij;\pi)-Q(\omegaij;\pi')}^2}{\sup_{\pi_a,\pi_b}\sumijwsq\bigpar{Q(\omegaij;\pi_a)-Q(\omegaij;\pi_b)}^2}\right\rceil \\
        &\le\sumij\rbra{\frac{m\cdot\fracwinisq\bigpar{Q(\omegaij;\pi)-Q(\omegaij;\pi')}^2}{\sup_{\pi_a,\pi_b}\sumijwsq\bigpar{Q(\omegaij;\pi_a)-Q(\omegaij;\pi_b)}^2}+1} \\
        &\le\frac{m\cdot\sumijwsq\bigpar{Q(\omegaij;\pi)-Q(\omegaij;\pi')}^2}{\sup_{\pi_a,\pi_b}\sumijwsq\bigpar{Q(\omegaij;\pi_a)-Q(\omegaij;\pi_b)}^2}+n\le m+n.
    \end{align*}
    Then, we have
    \begin{align*}
        \Ham(\pi,\pi';\tilde\omega)&=\frac{1}{\tilde{n}}\sumiM\sum_{i=1}^{\tilden_c}\ones\{\pi(\xij)\neq\pi'(\xij)\} \\
        &=\frac{1}{\tilden}\sumij\tildenij\cdot\ones\{\pi(\xij)\neq\pi'(\xij)\} \\
        &\ge\frac{1}{\tilden}\sumij\frac{m\cdot\fracwinisq\bigpar{Q(\omegaij;\pi)-Q(\omegaij;\pi')}^2}{\sup_{\pi_a,\pi_b}\sumijwsq\bigpar{Q(\omegaij;\pi_a)-Q(\omegaij;\pi_b)}^2}\ones\{\pi(\xij)\neq\pi'(\xij)\} \\
        &=\frac{m}{\tilden}\sumij\frac{\fracwinisq\bigpar{Q(\omegaij;\pi)-Q(\omegaij;\pi')}^2}{\sup_{\pi_a,\pi_b}\sumijwsq\bigpar{Q(\omegaij;\pi_a)-Q(\omegaij;\pi_b)}^2} \\
        &\ge\frac{m}{m+n}\ellw^2(\pi,\pi';\omega).
    \end{align*}
    
    This implies that
    \begin{equation*}
        \ellw(\pi,\pi';\omega)\le\sqrt{\frac{m+n}{m}\Ham(\pi,\pi';\tilde{\omega})}\le\sqrt{1+\frac{n}{m}}\cdot\epsilon.
    \end{equation*}
    Letting $m\to\infty$ yields $\ellw(\pi,\pi';\omega)\le\epsilon$. This establishes that for any $\pi\in\Pi$, there exists a $\pi'\in S_0$ such that $\ellw(\pi,\pi';\omega)\le\epsilon$, and thus $N_{\ellw}(\epsilon,\Pi;\omega)\le N_{\Ham}(\epsilon^2,\Pi)$.
\end{proof}

\subsubsection{Weighted Rademacher complexity}
Our learning bounds will rely on the following notion of \textit{weighted Rademacher complexity} introduced in \cite{mohri2019agnostic}.

\begin{definition}[Weighted Rademacher complexity]
    Suppose there is a set of clients $\calC$, with each client $c\in\calC$ having a data-generating distribution $\calPc$ defined over a space $\Omega$.
    Moreover, the clients have fixed sample sizes $\nC=(n_c)_{c\in\calC}$ and there is a distribution $\w$ over the set of clients $\calC$.
    For each client $c\in\calC$, let
    $W_1^c,\dots,W_{n_c}^c$ be independent random variables sampled from $\calPc$, and let $W=\{W_i^c\mid c\in\calC, i\in[n_c]\}$ represent the collection of samples across all clients.

    The \textit{empirical weighted Rademacher complexity} of a function class $\calF$ on $\Omega$ given multi-source data $W$ under mixture weights $\w$ and sample sizes $\nC$ is
    \begin{align*}
        \frakR_{\w,n_\calC}(\calF;W)\coloneqq\E\bra{\supF\biggabs{\sumijw\epsij f(\Wij)}\ \Big|\ W},
    \end{align*}
    where the expectation is taken with respect to the collection of independent Rademacher random variables $\eps=\{\epsij\mid c\in\calC, i\in[n_c]\}$.
    Additionally, the \textit{weighted Rademacher complexity} of $\calF$ under mixture weights $\w$ and sample sizes $n_\calC$ is
    \begin{align*}
        \frakR_{\w,\nC}(\calF)\coloneqq\E\bra{\supF\biggabs{\sumijw\epsij f(\Wij)}},
    \end{align*}
    where the expectation is taken with respect to the multi-source random variables $W$ and the independent Rademacher random variables $\eps$.
\end{definition}

\subsection{Client Heterogeneity}

\subsubsection{Client Distribution Skewness}\label{app:Preliminaries-ClientDistributionSkewness}

An important quantity that arises in our analysis is
\begin{equation*}
    \sumiM\fracwisqbarni=\opsE_{c\sim\w}\bra{\fracwibarni},
\end{equation*}
which captures a measure of the imbalance of the client distribution $\w$ relative to the empirical client distribution $\barn$. The following result makes this interpretation more clear:
\begin{align*}
    \sumiM\fracwisqbarni&=\sumiM\fracwisqbarni+\sumiM\barni-2\sumiM\w_c+1 \\
    &=\sumiM\rbra{\fracwisqbarni+\frac{\barni^2}{\barni}-\frac{2\wi\barni}{\barni}}+1 \\
    &=\sumiM\frac{(\wi-\barni)^2}{\barni} + 1 \\
    &=\chi^2(\w||\barn)+1.
\end{align*}
where $\chi^2(\w||\barn)$ is the chi-squared divergence from $\barn$ to $\w$.
Following \cite{mohri2019agnostic}, we call this quantity the skewness.
\begin{definition}[Skewness]
    The \textit{skewness} of a given distribution $\w$ over clients is
    \begin{equation*}
        \skewness\coloneqq\opsE_{c\sim\w}\bra{\fracwibarni}=1+\chi^2(\w||\barn),
    \end{equation*}
    where $\chi^2(\w||\barn)$ is the chi-squared divergence of $\w$ from $\barn$.
\end{definition}

\subsubsection{Client Distribution Shift}

In Section \ref{sec:Preliminaries-Data}, we defined $\calDcec$ to be the complete local data-generating distribution of client $c\in\calC$ and $\calDwew=\sumiM\wi\calDcec$ to the complete global data-generating distribution as a mixture of all other complete client data-generating distributions.
As we will observe in the local regret bounds, client heterogeneity will be captured by the distribution shift of the local distributions to the global distribution. In particular, the distribution shift of a local distribution $\calDcec$ from the global distribution $\calDwew$ will be captured by their total variation distance $\TV(\calDcec,\calDwew)$ and also their KL divergence $\KL(\calDcec||\calDwew)$. We will also introduce an alternate bound that will capture the distribution shift with their chi-squared divergence $\chi^2(\calDcec||\calDwew)$. See Appendix \ref{app:BoundingLocalRegret} for these results.

\section{Bounding Global Regret}\label{app:BoundingGlobalRegret}

\subsection{Preliminaries}

\subsubsection{Function Classes}
As mentioned in Appendix \ref{app:ellw2Distance}, the function class we will be considering in our analysis is
\begin{equation}\label{eq:FunctionClass}
    \calF_\Pi\coloneqq\{Q(\cdot;\pi):\Omega\to\R\mid\pi\in\Pi\},
\end{equation}
where
\begin{equation}\label{eq:QFunction}
    Q(\omega_i^c;\pi)\coloneqq\gij(\pi(\xij)),
\end{equation}
for any covariate-score vector $\omegaij=(\xij,\gij)\in\Omega=\calX\times\R^d$ and $\pi\in\Pi$, where $\gij(a)$ is the $a$-th coordinate of the score vector $\gij$. It will also be useful to consider the Minkowski difference of $\calF_\Pi$ with itself,
\begin{equation}\label{eq:DeltaFunctionClass}
    \Delta\calF_\Pi\coloneqq\calF_\Pi-\calF_\Pi=\{\Delta(\cdot;\pia,\pib):\Omega\to\R\mid\pia,\pib\in\Pi\},
\end{equation}
where
\begin{equation}\label{eq:DeltaQFunction}
    \Delta(\omegaij;\pia,\pib)\coloneqq Q(\omegaij;\pia)-Q(\omegaij;\pib)=\gij(\pia(\xij))-\gij(\pib(\xij)),
\end{equation}
for any $\omegaij=(\xij,\gij)\in\Omega$ and $\pia,\pib\in\Pi$.

\subsubsection{Policy Value Estimators}\label{app:Preliminaries-PolicyValueEstimators}

\paragraph{Augmented Inverse Propensity Weighted Scores}
We use propensity-weighted scores to estimate policy values. For any $c\in\calC$, consider the available observable samples $(\Xij,\Aij,\Yij)$ taken from the partially observable counterfactual sample $\Zij=(\Xij,\Aij,\Yij(a_1),\dots,\Yij(a_d))\sim\calDcec$, for $i\in[\nc]$. As discussed in Section \ref{sec:Approach-PolicyValueEstimators}, using this data, we considered construction of the oracle local AIPW scores
\begin{equation*}
    \Gij(a)=\mui(\Xij;a)+\bigpar{\Yij(\Aij)-\mui(\Xij;a)}\oi(\Xij;a)\ones\{\Aij=a\}
\end{equation*}
for each $a\in\calA$. Similarly, we discussed the construction of the approximate local AIPW scores
\begin{equation*}
    \Ghatij(a)=\muhati(\Xij;a)+\bigpar{\Yij(\Aij)-\muhati(\Xij;a)}\ohat(\Xij;a)\ones\{\Aij=a\}
\end{equation*}
for each $a\in\calA$, given fixed estimates $\muhati$ and $\ohati$ of $\mui$ and $\oi$, respectively. In practice, we use cross-fitting to make the estimates fixed and independent relative to the data on which they are evaluated. Note that only this second set of scores can be constructed from the observed data. The first set is ``constructed" for analytic purposes in our proofs.

\paragraph{Policy Value Estimates and Policy Value Difference Estimates}

Using the local data and the constructed AIPW scores, we let $\Wij=(\Xij,\Gij(a_1),\dots,\Gij(a_d))$ and $\Whatij=(\Xij,\Ghatij(a_1),\dots,\Ghatij(a_d))$ for each $i\in[\nc]$. We define the oracle and approximate \textit{policy value} estimates of $\Qw(\pi)$, respectively, as
\begin{align*}
    \Qtildew(\pi)&=\sumijw\Gij(\pi(\Xij))=\sumijw Q(\Wij;\pi), \\
    \Qhatw(\pi)&=\sumijw\Ghatij(\pi(\Xij))=\sumijw Q(\Whatij;\pi),
\end{align*}
for any $\pi\in\Pi$, where we use the function class defined in Equation \ref{eq:QFunction} for the alternate representations that we will use throughout our proofs for notational convenience. It will also be very useful to define the following corresponding \textit{policy value difference} quantities:
\begin{align*}
    \Deltaw(\pia,\pib)&\coloneqq\Qw(\pia)-\Qw(\pib), \\
    \Deltatildew(\pia,\pib)&\coloneqq\Qtildew(\pia)-\Qtildew(\pib)=\sumijw\Delta(\Wij;\pia,\pib), \\
    \Deltahatw(\pia,\pib)&\coloneqq\Qhatw(\pia)-\Qhatw(\pib)=\sumijw\Delta(\Whatij;\pia,\pib),
\end{align*}
for any $\pia,\pib\in\Pi$, where we use the function class defined in Equation \eqref{eq:DeltaQFunction} for the alternate representations that we will use throughout our proofs for notational convenience.

\paragraph{Unbiased Estimates}
The following result can be used to readily show that the oracle estimators for the local and global policy values are unbiased.
\begin{lemma}\label{lem:ExpectedOracleEqualsLocalPolicyValue}
    Suppose Assumption \ref{ass:dgp} holds.
    For any $\pi\in\Pi$,
    \begin{equation*}
        \opsE_{\Zi\sim\calDcec}\regbra{\Gi(\pi(\Xi))}=\Qi(\pi)
    \end{equation*}
    and
    \begin{equation*}
        \opsE_{c\sim\w}\opsE_{\Zi\sim\calDcec}[\Gi(\pi(\Xi))]=\Qw(\pi),
    \end{equation*}
    where $\Zi=(\Xi,\Ai,\Yi(a_1),\dots,\Yi(a_d))\sim\calDcec$.
\end{lemma}
\begin{proof}
    First, observe that for any $a\in\calA$,
    \begin{align*}
        \Gi(a)&=\mui(\Xi;a)+\rbra{\Yi(\Ai)-\mui(\Xi;a)}\oi(\Xi;a)\ones\{\Ai=a\} \\
        &=\mui(\Xi;a)+\rbra{\Yi(a)-\mui(\Xi;a)}\oi(\Xi;a)\ones\{\Ai=a\}
    \end{align*}
    due to the indicator in the definition,
    and so for any $(\Xi,\Ai,\Yi(a_1),\dots,\Yi(a_d))\sim\calDcec$,
    \begin{align*}
        \E_{\Ai,\vecYi}\bra{\Gi(a)\mid\Xi}&=\mui(\Xi;a)+\E_{\Ai,\vecYi}\bra{\rbra{\Yi(a)-\mui(\Xi;a)}\oi(\Xi;a)\ones\{\Ai=a\}\mid\Xi}  \\
        &=\mui(\Xi;a)+\E_{\vecYi}\bra{\Yi(a)-\mui(\Xi;a)\mid\Xi}\cdot\E_{\Ai}\bra{\oi(\Xi;a)\ones\{\Ai=a\}\mid\Xi} \\
        &=\mui(\Xi;a)+\E_{\vecYi}\bra{\Yi(a)-\mui(\Xi;a)\mid\Xi}\cdot\oi(\Xi;a)\ei(\Xi;a) \\
        &=\mui(\Xi;a)+\E_{\vecYi}\bra{\Yi(a)\mid\Xi}-\mui(\Xi;a) \\
        &=\E_{\vecYi}\bra{\Yi(a)\mid\Xi}.
    \end{align*}
    The second equality follows from the unconfoundedness assumption stated in Assumption \ref{ass:dgp}. This immediately implies that
    \begin{align*}
        \opsE_{\Zi\sim\calDcec}\regbra{\Gi(\pi(\Xi))}
        &=\opsE_{\Zi\sim\calDcec}\bra{\sumaA\ones\{\pi(\Xi)=a\}\Gi(a)} \\
        &=\E_{\Xi}\bra{\sumaA\ones\{\pi(\Xi)=a\}\E_{\Ai,\vecYi}\bra{\Gi(a)\mid\Xi}} \\
        &=\E_{\Xi}\bra{\sumaA\ones\{\pi(\Xi)=a\}\E_{\Ai,\vecYi}\bra{\Yi(a)\mid\Xi}} \\
        &=\E_{\Xi}\bra{\E_{\Ai,\vecYi}\bra{\Yi(\pi(\Xi))\mid\Xi}} \\
        &=\opsE_{\Zi\sim\calDcec}\bra{\Yi(\pi(\Xi)} \\
        &=\Qi(\pi),
    \end{align*}
    and
    \begin{align*}
        \opsE_{c\sim\w}\opsE_{\Zi\sim\calDcec}\regbra{\Gi(\pi(\Xi))}=\opsE_{c\sim\w}\bra{\Qi(\pi)}=\Qw(\pi).
    \end{align*}
\end{proof}

\subsubsection{Data-generating Distributions and Sufficient Statistics}\label{app:BoundingGlobalRegret-Preliminaries-DGP&SS}

As introduced in the problem setting in Section \ref{sec:Preliminaries-Setting}, each client has a data-generating distribution $\Di$ defined over the joint space $\calX\times\calY^d$ of contexts and potential outcomes.
Moreover, the historical policy $\ei:\calX\to\Delta(\calA)$ induces the complete data-generating distribution $\calDcec$ defined over the joint space $\calX\times\calA\times\calY^d$ of contexts, actions, and potential outcomes
such that sampling $(\Xi,\Ai,\Yi(a_1),\dots,\Yi(a_d))\sim\calDcec$ is defined as sampling $(\Xi,\Yi(a_1),\dots,\Yi(a_d))\sim\calDc$ and $\Ai\sim\ei(\cdot|\Xi)$.

Note that, by construction, the contexts and AIPW scores are sufficient statistics for the corresponding oracle and approximate estimators of the policy values. Moreover, our results will mostly depend on properties of the sufficient statistics (e.g., AIPW score range and variance). Therefore, it will be useful for notational simplicity in our analysis to define the distribution of the sufficient statistics. For any $\Zi=(\Xi,\Ai,\Yi(a_1),\dots,\Yi(a_d))\sim\calDcec$, let
\begin{align*}
    \Wi=(\Xi,\Gi(a_1),\dots,\Gi(a_d))
\end{align*}
be the sufficient statistic of contexts and oracle AIPW scores, and we denote its induced distribution as $\tildecalDc$ defined over $\Omega=\calX\times\R^d$.

\begin{remark}
    For simplicity, without loss of generality, when proving results that only involve the contexts and AIPW scores, we will assume the data is sampled from the distributions of the sufficient statistics, e.g., $\Wi\sim\tildecalDc$. When we have a discussion involving constructing the AIPW scores from the observable data, we will be more careful about the source distributions and typically assume the data is sampled from the complete data-generating distributions, e.g., $\Zi\sim\calDcec$.
\end{remark}

\subsubsection{Proof Sketch}\label{app:ProofSketch}

We describe our general proof strategy.
The standard approach for proving finite-sample regret bounds in offline policy learning is to establish uniform concentration bounds around a proper notion of empirical complexity, which is then further bounded by class-dependent vanishing rates \citep{athey2021policy, zhou2023offline}. Typically, this complexity notion involves the Rademacher complexity of an appropriate policy value-based function class. However, this is not applicable to our scenario where the data may not come from the same source distribution. In our proof, we draw inspiration from the work on empirical risk bounds in multiple-source supervised learning settings, particularly \citep{mohri2019agnostic}, to identify the suitable notion of complexity—namely, the weighted Rademacher complexity of the policy value function class $\calF_\Pi$.
While \cite{mohri2019agnostic} provided a starting framework for a multiple-source analysis in supervised learning, the proof techniques for establishing class-dependent uniform concentration results in offline policy learning are typically more involved than those for empirical risk bounds in supervised learning. Our bounds necessitate more complex Dudley-type chaining arguments with applications of Talagrand’s inequalities with some multiple-source modifications mediated by skewness.

To begin, we split the global regret in terms of an oracle regret term of policy value differences and an approximation error term.
Let $\pistarw=\argmax_{\pi\in\Pi}\Qw(\pi)$. We can decompose the regret incurred by the global policy $\pihatw=\argmax_{\pi\in\Pi}\Qhatw(\pi)$ as follows:
\begin{align*}
    R_\w(\pihatw)&=\Qw(\pistarw)-\Qw(\pihatw) \\
    &=\bigpar{\Qw(\pistarw)-\Qw(\pihatw)} - \bigpar{\Qhatw(\pistarw)-\Qhatw(\pihatw)} + \bigpar{\Qhatw(\pistarw)-\Qhatw(\pihatw)} \\
    &=\Deltaw(\pistarw,\pihatw) - \Deltahatw(\pistarw,\pihatw) + \bigpar{\Qhatw(\pistarw)-\Qhatw(\pihatw)} \\
    &\le\Deltaw(\pistarw,\pihatw) - \Deltahatw(\pistarw,\pihatw) \\
    &\le\supPiab\regabs{\Deltaw(\pia,\pib)-\Deltahatw(\pia,\pib)} \\
    &\le\supPiab\regabs{\Deltaw(\pia,\pib)-\Deltatildew(\pia,\pib)}+\supPiab\regabs{\Deltatildew(\pia,\pib)-\Deltahatw(\pia,\pib)}.
\end{align*}
The first inequality holds by definition of the global policy, the second inequality by a straightforward worst-case supremum bound, and the last inequality by the triangle inequality.

The first oracle term in the last inequality will be bounded by the weighted Rademacher complexity of $\Delta\calF_\Pi$. Then, using a Dudley chaining argument, the weighted Rademacher complexity will be bounded by a measure of policy class complexity and vanishing rates with respect to the total sample size. We emphasize that establish rates with respect to the total sample size, rather than some other more moderate quantity of the sample sizes, such as the average or the minimum.

The second term in the last inequality will be bounded by a decomposition of the approximation terms which will be shown to be asymptotically vanishing faster than the rate bounds of the oracle regret with high probability. Establishing these bounds requires the use of Assumption \ref{ass:LocalDataSizeScaling} to ensure there is enough data across clients. Altogether, these results will provide a rate bound on the global regret that scales with the total sample size and is mediated by client skewness.

In a later section, we establish bounds for the notion of local regret, unique to our problem setting. This insight arises from recognizing the mismatch between global server-level performance and local client-level performance. We derive a local regret bound dependent on measures of distribution shift between clients, providing valuable insights into the value of information in heterogeneous client participation and how exactly heterogeneity affects policy performance for any given client. This exact quantification is highlighted in our Theorem \ref{thm:SubMainTheorem} that decomposes the sources of heterogeneity at the population, environment, and treatment level. We also point to Theorem \ref{thm:SubSubMainTheorem} for an alternative local regret bound that does not require bounded inverse propensity weighted scores.

\subsection{Bounding Weighted Rademacher Complexity}

First, to simplify our analysis, we can easily bound the weighted Rademacher complexity of $\Delta\calF_\Pi$ by that of $\calF_\Pi$ as follows.
\begin{lemma}\label{lem:RademacherComplexityDifferenceBound}
    \begin{align*}
        \frakR_{\w,\nC}(\Delta\calF_\Pi)\le2\frakRwnF
    \end{align*}
\end{lemma}
\begin{proof}
    \begin{align*}
        \frakR_{\w,\nC}(\Delta\calF_\Pi)&=\E\bra{\supPiab\Biggabs{\sumijw\epsij\Delta(\Wij;\pia,\pib)}} \\
        &=\E\bra{\supPiab\Biggabs{\sumijw\epsij\Bigpar{Q(\Wij;\pia)-Q(\Wij;\pib)}}} \\
        &\le\E\bra{\supPiab\Biggabs{\sumijw\epsij Q(\Wij;\pia)}+\Biggabs{\sumijw\epsij Q(\Wij;\pib)}} \\
        &=2\E\bra{\supPi\Biggabs{\sumijw\epsij Q(\Wij;\pi)}} \\
        &=2\frakR_{\w,\nC}(\calF_\Pi).
    \end{align*}
\end{proof}

Therefore, we can simply focus on bounding the weighted Rademacher complexity of $\calF_\Pi$.

\begin{proposition}\label{prop:WeightedRademacherBound}
Suppose Assumptions \ref{ass:dgp} and \ref{ass:LocalDataSizeScaling} hold. Then,
\begin{equation*}
    \frakR_{\w,\nC}(\calF_\Pi)\le\rbra{14+6\kappa(\Pi)}\sqrtfracVwn + \littleo{\sqrt{\frac{\skewness}{n}}},
\end{equation*}
where
\begin{equation*}
    \Vwn=\supPiab\sumiM\fracwisqbarni\opsE_{\Wi\sim\tilde{\calD}_c}\bra{\Delta^2(\Wi;\pia,\pib)}.
\end{equation*}
\end{proposition}
\begin{proof}

We follow a chaining argument to bound the weighted Rademacher complexity of $\calF_\Pi$.
\vspace{-.5em}
\paragraph{Constructing the policy approximation chain.}
First, for each client $c\in\calC$, let $W_1^c,\dots,W_{n_c}^c$ be $\nc$ independent random variables sampled from $\tildecalDc$, where each $\Wij=(\Xij,\vecGij)\in\Omega=\calX\times\R^d$. Additionally, let $W=\{\Wij\mid c\in\calC, i\in[n_c]\}$ represent the corresponding collection of samples across all clients.

Next, set $K=\lceil\log_2 n\rceil$. We will construct a sequence $\{\Psi_k:\Pi\to\Pi\}_{k=0}^K$ of policy approximation operators that satisfies the following properties. For any $k=0,\dots,K$,
\begin{itemize}
    \item[(P1)] $\maxPi\ellw(\Psi_{k+1}(\pi),\Psi_k(\pi);Z)\le \epsilon_k\coloneqq 2^{-k}$
    
    \item[(P2)] $\abs{\{\Psi_k(\pi)\mid\pi\in\Pi\}}\le N_{\ellw}(\epsilon_k,\Pi;Z)$
\end{itemize}
We use the notational shorthand that $\Psi_{K+1}(\pi)=\pi$ for any $\pi\in\Pi$. We will construct the policy approximation chain via a backward recursion scheme. First, let $\Pi_k$ denote the smallest $\epsilon_k$-covering set of $\Pi$ under the $\ellw$ distance given data $Z$. Note, in particular, that $\abs{\Pi_0}=1$ since the $\ellw$ distance is never more than 1 and so any single policy is enough to 1-cover all policies in $\Pi$. Then, the backward recursion is as follows: for any $\pi\in\Pi$,
\begin{enumerate}
    \item define $\Psi_K(\pi)=\argmin_{\pi'\in\Pi_K}\ellw(\pi,\pi';W)$;
    
    \item for each $k=K-1,\dots,1$, define $\Psi_k(\pi)=\argmin_{\pi'\in\Pi_k}\ellw(\Psi_{k+1}(\pi),\pi';W)$;
    
    \item define $\Psi_0(\pi)\equiv 0$.
\end{enumerate}

Note that although $\Psi_0(\pi)$ is not in $\Pi$, it can still serve as a $1$-cover of $\Pi$ since the $\ellw$ distance is always bounded by 1. Before proceeding, we check that each of the stated desired properties of the constructed operator chain is satisfied:
\begin{itemize}
    \item[(P1)] Pick any $\pi\in\Pi$. Clearly, $\Psi_{k+1}(\pi)\in\Pi$. Then, by construction of $\Pi_k$, there exists a $\pi'\in\Pi_k$ such that $\ellw(\Psi_{k+1}(\pi),\pi';W)\le\epsilon_k$. Therefore, by construction of $\Psi_k(\pi)$, we have $\ellw(\Psi_{k+1}(\pi),\Psi_k(\pi);W)\le\ellw(\Psi_{k+1}(\pi),\pi';W)\le\epsilon_k$.
    
    \item[(P2)] By construction of $\Psi_k$, we have that $\Psi_k(\pi)\in\Pi_k$ for every $\pi\in\Pi$. Therefore, $\abs{\{\Psi_k(\pi)\mid\pi\in\Pi\}}\le\abs{\Pi_k}=N_{\ellw}(\epsilon_k,\Pi;W)$.
\end{itemize}

Thus, the constructed chain satisfies the desired properties.
Next, we observe that since $\Psi_0(\pi)\equiv 0$, we have that $Q(\Wij;\Psi_0(\pi))=0$ and
\begin{align*}
    Q(\Wij;\pi)&=Q(\Wij;\pi)-Q(\Wij;\Psi_0(\pi)) \\
    &=Q(\Wij;\pi)-Q(\Wij;\Psi_K(\pi)) + \sumkK Q(\Wij;\Psi_k(\pi)-Q(\Wij;\Psi_{k-1}(\pi)) \\
    &=\Delta(\Wij;\pi,\Psi_K(\pi)) + \sumkK\Delta(\Wij;\Psi_k(\pi),\Psi_{k-1}(\pi))
\end{align*}
Therefore, we can decompose the weighted Rademacher complexity of $\calF_\Pi$ as follows:
\begin{align*}
    \frakRwnF=&\E\bra{\supPi\Biggabs{\sumijw\epsij\Delta(\Wij;\pi,\Psi_K(\pi))}} \\
    &+\E\bra{\supPi\Biggabs{\sumijw\epsij\biggpar{\sumkK \Delta(\Wij;\Psi_k(\pi),\Psi_{k-1}(\pi))}}}
\end{align*}
We will obtain bounds separately for these two terms, which we refer to as the \textit{negligible regime} term and the \textit{effective regime} term, respectively.

\paragraph{Bounding the negligible regime.}
For convenience, we denote
\begin{equation*}
    \BwW\coloneqq\sup_{\pi_a,\pi_b\in\Pi}\sumij\fracwinisq\Delta^2(\Wij;\pia,\pib)
\end{equation*}
and $\Bw\coloneqq\E\bra{\BwW}$.

\vspace{.25em}
Given any realization of independent Rademacher random variables $\eps=\{\epsij\mid c\in\calC, i\in[\nc]\}$ and multi-source data $W$, by the Cauchy-Schwarz inequality,
\begin{align*}
    &\biggabs{\sumijw\epsij\Delta(\Wij;\pi,\Psi_K(\pi)} \\
    &\le\sqrt{\sumij(\epsij)^2}\cdot\sqrt{\sumij\fracwinisq\Delta^2(\Wij;\pi,\Psi_K(\pi)} \\
    &=\sqrt{n}\cdot\sqrt{\BwW}\ell_2(\pi,\Psi_K(\pi);Z) \\
    &\le\sqrt{n\BwW}\epsilon_k \\
    &\le\sqrt{\frac{\BwW}{n}}.
\end{align*}
Then, by Jensen's inequality,
\begin{align*}
    \E\bra{\supPi\biggabs{\sumijw\epsij\Delta(\Wij;\pi,\Psi_K(\pi)}}&\le\E\bra{\sqrt{\frac{\BwW}{n}}}\le\sqrt{\frac{\Bw}{n}}.
\end{align*}

\paragraph{Bounding the effective regime.}

For any $k\in[K]$, let $$t_{k,\delta}=\sqrt{\BwW}\epsilon_k\tau_{k,\delta}$$ where $\tau_{k,\delta}>0$ is some constant to be specified later. By Hoeffding's inequality (in Lemma \ref{lem:HoeffdingInequality}),
\begin{align*}
    &\P\rbra{\biggabs{\sumijw\epsij\Delta(\Wij;\Psi_k(\pi),\Psi_{k-1}(\pi))}>t_{k,\delta}\Bigmid W} \\
    &\le 2\exp\rbra{-\frac{t_{k,\delta}^2}{2\sumij\fracwinisq\Delta^2(\Wij;\Psi_k(\pi),\Psi_{k-1}(\pi))}} \\
    &=2\exp\rbra{-\frac{t_{k,\delta}^2}{2\BwW\ell_2^2(\Psi_k(\pi),\Psi_{k-1}(\pi);W)}} \\
    &\le2\exp\rbra{-\frac{t_{k,\delta}^2}{2\BwW\epsilon_{k-1}^2}} \\
    &=2\exp\rbra{-\frac{t_{k,\delta}^2}{8\BwW\epsilon_k^2}} \\
    &=2\exp\rbra{-\frac{\tau_{k,\delta}^2}{8}}.
\end{align*}
Here, we used the fact that $\epsilon_{k-1}=2\epsilon_k$. Setting
\begin{equation*}
    \tau_{k,\delta}=\sqrt{8\log\rbra{\frac{\pi^2k^2}{3\delta}N_{\ellw}(\epsilon_k,\Pi;W)}}
\end{equation*}
and applying a union bound over the policy space, we obtain
\begin{align*}
    &\P\rbra{\supPi\biggabs{\sumijw\epsij\Delta(\Wij;\Psi_k(\pi),\Psi_{k-1}(\pi))}>t_{k,\delta}\Bigmid W} \\
    &\le2\abs{\Pi_k}\cdot\exp\rbra{-\frac{\tau_{k,\delta}^2}{8}} \\
    &\le2N_{\ell_2}(\epsilon_k,\Pi;W)\cdot\exp\rbra{-\frac{\tau_{k,\delta}^2}{8}} \\
    &=\frac{6\delta}{\pi^2k^2}.
\end{align*}
By a further union bound over $k\in[K]$, we obtain
\begin{align*}
    &\P\rbra{\supPi\biggabs{\sumijw\epsij\biggpar{\sumkK \Delta(\Wij;\Psi_k(\pi),\Psi_{k-1}(\pi))}}>\sumkK t_{k,\delta}\Bigmid W} \\
    &\le\sumkK\P\rbra{\supPi\biggabs{\sumijw\epsij\Delta(\Wij;\Psi_k(\pi),\Psi_{k-1}(\pi))}>t_{k,\delta}\Bigmid W} \\
    &\le\sumkK\frac{6\delta}{\pi^2k^2} \le\delta.
\end{align*}
Therefore, given multi-source data $W$, with probability at least $1-\delta$, we have
\begin{align*}
    &\supPi\Biggabs{\sumijw\epsij\biggpar{\sumkK \Delta(\Wij;\Psi_k(\pi),\Psi_{k-1}(\pi))}} \\
    &\le\sumkK t_{k,\delta} \\
    &=\sqrt{\BwW}\sumkK\epsilon_k\sqrt{8\log\rbra{\frac{\pi^2k^2}{3\delta}N_{\ell_2}(\epsilon_k,\Pi;W)}} \\
    &=\sqrt{\BwW}\sumkK\epsilon_k\rbra{\sqrt{8\log\frac{\pi^2}{3\delta}+16\log k + 8\log N_{\ell_2}(\epsilon_k,\Pi;W)}} \\
    &\le\sqrt{\BwW}\sumkK\epsilon_k\rbra{\sqrt{8\log\frac{\pi^2}{3\delta}}+\sqrt{16\log k} + \sqrt{8\log N_{\ell_2}(\epsilon_k,\Pi;W)}} \\
    &\le\sqrt{\BwW}\sumkK\epsilon_k\rbra{\sqrt{8\log(4/\delta)}+\sqrt{16\log k} + \sqrt{8\log N_{\Ham}(\epsilon_k,\Pi)}} \\
    &\le\sqrt{\BwW}\rbra{\sqrt{8\log(4/\delta)}+2+\sqrt{8}\sum_{k=1}^\infty\epsilon_k\sqrt{\log N_{\Ham}(\epsilon_k,\Pi)}} \\
    &\le\sqrt{\BwW}\rbra{\sqrt{8\log(4/\delta)}+2+\sqrt{8}\kappa(\Pi)}.
\end{align*}
Next, we turn this high-probability bound into a bound on the conditional expectation. First, let $F_R(\cdot\mid W)$ be the cumulative distribution of the random variable
\begin{equation*}
    R\coloneqq \supPi\Biggabs{\sumijw\epsij\biggpar{\sumkK \Delta(\Wij;\Psi_k(\pi),\Psi_{k-1}(\pi))}}
\end{equation*}
conditional on $W$. Above, we have shown that
\begin{equation*}
    1-F_R\rbra{\sqrt{\BwW}\rbra{\sqrt{8\log(4/\delta)}+2+\sqrt{8}\kappa(\Pi)}\ \big|\ W}\le\delta.
\end{equation*}

For any non-negative integer $l$, let $\Delta_l=\sqrt{\BwW}(\sqrt{8\log(4/\delta_l)}+2+\sqrt{8}\kappa(\Pi))$ where $\delta_l=2^{-l}$.
Since $R$ is non-negative, we can compute and upper bound the conditional expectation of $R$ given $W$ as follows:
\begin{align*}
    &\E\bra{\supPi\Biggabs{\sumijw\epsij\biggpar{\sumkK \Delta(\Wij;\Psi_k(\pi),\Psi_{k-1}(\pi))}}\Bigmid W} \\
    &=\int_0^\infty\rbra{1-F_R(r|W)}dr \\
    &\le\sum_{l=0}^\infty\rbra{1-F_R(\Delta_l|W)}\Delta_l \\
    &\le\sum_{l=0}^\infty\delta_l\Delta_l \\
    &=\sum_{l=0}^\infty 2^{-l}\cdot \sqrt{\BwW}\rbra{\sqrt{8(l+2)\log 2}+2+\sqrt{8}\kappa(\Pi)} \\
    &\le\sqrt{\BwW}\rbra{4\sqrt{8\log 2}+4+2\sqrt{8}\kappa(\Pi)} \\
    &\le\sqrt{\BwW}\rbra{14+6\kappa(\Pi)}.
\end{align*}
Taking the expectation with respect to $W$ and using Jensen's inequality, we obtain
\begin{align*}
    &\E\bra{\supPi\Biggabs{\sumijw\epsij\biggpar{\sumkK \Delta(\Wij;\Psi_k(\pi),\Psi_{k-1}(\pi))}}} \\
    &\le\rbra{14+6\kappa(\Pi)}\E\bra{\sqrt{\BwW}} \\
    &\le\rbra{14+6\kappa(\Pi)}\sqrt{\Bw}.
\end{align*}

\paragraph{Refining the upper bound.}
One could easily bound $\Bw$ using worst-case bounds on the AIPW element. Instead, we use Lemma \ref{lem:ExpectedSupSumBound} to get a more refined bound on $\Bw$.

To use this result, we identify the set of independent random variables $\tilde\Wij=T(\Wij)=(\Xij,\fracwini\Gij)$ for $c\in\calC$ and $i\in[\nc]$ and the function class $\calH=\{\Delta(\cdot;\pia,\pib)\mid\pia,\pib\in\Pi\}$. We also identify the Lipschitz function $\varphi:u\mapsto u^2$ defined over the set $\calU\subset\R$ containing all possible outputs of any $\Delta(\cdot;\pia,\pib)$ given any realization of $\tilde\Wij$ for any $c\in\calC$ as input. To further capture this domain, note that by the boundedness and overlap assumptions in Assumption \ref{ass:dgp}, it is easy to verify that there exists some $U>0$ for all $c\in\calC$ such that $|\Gij(a)|\le U$ for any $a\in\calA$ and any realization of $\Wij=(\Xij,\Gij)$. This implies that
\begin{align*}
    \regabs{\Delta(\tilde\Wij;\pia,\pib)}&=\fracwini\bigabs{\Gij(\pia(\Xij))-\Gij(\pib(\Xij))}\le 2U\fracwini,
\end{align*}
for any realization of $\tilde\Wij$ and any $\pia,\pib\in\Pi$. Moreover, note that
\begin{align}
    \fracwini&\le\max_{c\in\calC}\fracwini\le\sqrt{\sumiM\fracwinisq}\le\frac{1}{\sqrt{\min_{c\in\calC}\nc}}\sqrt{\sumiM\fracwisqni}=\frac{1}{\sqrt{\min_{c\in\calC}\nc}}\sqrtfracskewnessn\eqqcolon\swnC \label{eq:ClientDistributionRatioBound}
\end{align}
Therefore, $\calU\subset[-\swnC,\swnC]$, and thus, for any $u,v\in\calU$, we have that
\begin{align*}
    \regabs{\varphi(u)-\varphi(v)}&=\regabs{u^2-v^2}=\abs{u+v}\cdot\abs{u-v}\le4U\swnC\abs{u-v}.
\end{align*}
Therefore, $L=4U\swnC$ is a valid Lipschitz constant for $\varphi$. Then, through these identifications, Lemma \ref{lem:ExpectedSupSumBound} guarantees the following upper bound
\begin{align*}
    \Bw&=\E\bra{\supPiab\sumij\fracwinisq\Delta^2(\Wij;\pia,\pib)} \\
    &=\E\bra{\supPiab\sumij\varphi\circ\Delta(\tilde\Wij;\pia,\pib)} \\
    &\le\supPiab\sumij\E\bra{\varphi\circ\Delta(\tilde\Wij;\pia,\pib)} + 16U\swnC\E\bra{\supPiab\biggabs{\sumij\epsij\Delta(\tilde\Wij;\pia,\pib)}} \\
    &=\supPiab\sumij\fracwinisq\E\bra{\Delta^2(\Wij;\pia,\pib)} + 16U\swnC\E\bra{\supPiab\biggabs{\sumijw\epsij\Delta(\Wij;\pia,\pib)}} \\
    &=\supPiab\sumiM\fracwisqni\E\bra{\Delta^2(\Wij;\pia,\pib)} + 16U\frakRwn(\Delta\calF_\Pi)\swnC \\
    &\le\supPiab\sumiM\fracwisqni\E\bra{\Delta^2(\Wij;\pia,\pib)} + 32U\frakRwnF\swnC \\
    &=\fracVwn + 32U\frakRwnF\swnC.
\end{align*}

Before proceeding, note that by the local data size scaling assumption stated in Assumption \ref{ass:LocalDataSizeScaling}, $\nc=\Omega(\nu_c(n))$ for some increasing function $\nu_c$ for any $c\in\calC$. This immediately implies that $\swnC$ is dominated as
\begin{align*}
    \swnC=\frac{1}{\sqrt{\min_{c\in\calC}\nc}}\sqrtfracskewnessn\le\littleo{\sqrtfracskewnessn}.
\end{align*}

\paragraph{Combine results.}
Thus, combining the bounds for the negligible and effective regime and including the refined bound, we have
\begin{align*}
    &\frakRwnF \\
    &\le\sqrt{\frac{\Bw}{n}}+\rbra{14+6\kappa(\Pi)}\sqrt{\Bw} \\
    &\le\sqrt{\frac{\Vwn}{n^2}+32U\frakRwnF\frac{\swnC}{n}} + (14+6\kappa(\Pi))\sqrt{\fracVwn + 32U\frakRwnF\swnC} \\
    &\le\sqrt{\frac{\Vwn}{n^2}} + \sqrt{32U\frakRwnF\frac{\swnC}{n}} + (14+6\kappa(\Pi))\rbra{\sqrtfracVwn + \sqrt{32U\frakRwnF\swnC}} \\
    &\le(14+6\kappa(\Pi))\sqrtfracVwn + \sqrt{\frac{\Vwn}{n^2}} + \bigO{\sqrt{\frakRwnF\swnC}} \numberthis \label{eq:prop-WeightedRademacherBound-eq1}
\end{align*}
This gives an upper bound on $\frakRwnF$ in terms of itself. To decouple this dependence, we express
\begin{align*}
    \frakRwnF&\le \bigO{\sqrtfracVwn} + \bigO{\sqrt{\frakRwnF\swnC}} \\
    &\le A_1\sqrtfracVwn + A_2\sqrt{\frakRwnF\swnC} \numberthis \label{eq:prop-WeightedRademacherBound-eq2}
\end{align*}
for some constants $A_1, A_2$, and we split this inequality into the following two exhaustive cases.

\textit{\underline{Case 1}}: $A_2\sqrt{\swnC}\le\frac{1}{2}\sqrt{\frakRwnF}$

In this case, we can bound the second term in the right-hand side of inequality \eqref{eq:prop-WeightedRademacherBound-eq2} to get
\begin{align*}
    \frakRwnF\le A_1\sqrtfracVwn + \frac{\frakRwnF}{2},
\end{align*}
and so
\begin{align}\label{eq:prop-WeightedRademacherBound-eq3}
    \frakRwnF\le 2A_1\sqrtfracVwn.
\end{align}
Moreover,
\begin{align*}
    \Vwn&=\supPiab\sumiM\fracwisqbarni\opsE_{\Wi\sim\tildecalDc}\regbra{\Delta^2(\Wi;\pia,\pib)} \\
    &\le\supPiab\max_{c\in\calC}\opsE_{\Wi\sim\calPc}\regbra{\Delta^2(\Wi;\pia,\pib)}\cdot\sumiM\fracwisqbarni =\barV\skewness,
\end{align*}
where $\barV=\supPiab\max_{c\in\calC}\opsE_{\Wi\sim\calPc}\bra{\Delta(\Wi;\pia,\pib)}$, which is a constant value. Note that the last equality holds by the skewness identity in established in Appendix \ref{app:Preliminaries-ClientDistributionSkewness}. Plugging this into inequality \eqref{eq:prop-WeightedRademacherBound-eq3}, we get
\begin{align*}
    \frakRwnF\le2A_1\sqrt{\frac{\barV\skewness}{n}}\le\bigO{\sqrtfracskewnessn}.
\end{align*}

\textit{\underline{Case 2}}: $A_2\sqrt{\swnC}>\frac{1}{2}\sqrt{\frakRwnF}$

In this case, one can easily rearrange terms to get that
\begin{align*}
    \frakRwnF<4A_2^2\swnC\le\littleo{\sqrtfracskewnessn}.
\end{align*}

\vspace{.5em}
Therefore, in either case, $\frakRwnF\le\calO\Bigpar{\sqrtfracskewnessn\,}$. We can plug this asymptotic bound into inequality \eqref{eq:prop-WeightedRademacherBound-eq1} to arrive at the desired result,
\begin{align*}
    \frakRwnF&\le \rbra{14+6\kappa(\Pi)}\sqrtfracVwn + \sqrt{\frac{\Vwn}{n^2}} + \bigO{\sqrt{\bigO{\sqrtfracskewnessn}\swnC}} \\
    &\le \rbra{14+6\kappa(\Pi)}\sqrtfracVwn + \sqrt{\frac{\barV\skewness}{n^2}} + \bigO{\sqrt{\bigO{\sqrtfracskewnessn}\littleo{\sqrtfracskewnessn}}} \\
    &\le \rbra{14+6\kappa(\Pi)}\sqrtfracVwn + \littleo{\sqrtfracskewnessn} + \littleo{\sqrtfracskewnessn} \\
    &\le \rbra{14+6\kappa(\Pi)}\sqrtfracVwn + \littleo{\sqrtfracskewnessn}.
\end{align*}
\end{proof}

\subsection{Bounding Oracle Regret}

\begin{proposition}\label{prop:OracleRegretBound}
Suppose Assumptions \ref{ass:dgp} and \ref{ass:LocalDataSizeScaling} hold.
Then, with probability at least $1-\delta$,
\begin{equation*}
    \supPiab|\Deltaw(\pia,\pib)-\Deltatildew(\pia,\pib)|\le\rbra{c_1\kappa(\Pi) + \sqrt{c_2\log(c_2/\delta)}}\sqrtfracVwn + \littleo{\sqrtfracskewnessn},
\end{equation*}
where $c_1$ and $c_2$ are universal constants.
\end{proposition}
\begin{proof}
    First, for each client $c\in\calC$, let $W_1^c,\dots,W_{n_c}^c$ be $\nc$ independent random variables sampled from $\tildecalDc$, where each $\Wij=(\Xij,\vecGij)\in\Omega=\calX\times\R^d$. Additionally, let $W=\{\Wij\mid c\in\calC, i\in[n_c]\}$ represent the corresponding collection of samples across all clients. 
    
    In Lemma \ref{lem:ExpectedOracleEqualsLocalPolicyValue}, we showed that $\E_{\Wi\sim\tildecalDc}\bra{Q(\Wi;\pi)}=\Qi(\pi)$.    
    This implies that
    \begin{align*}
        \E_W\regbra{\Qtilde_\w(\pi)}&=\sumijw\opsE\bra{Q(\Wij;\pi)}=\sumijw\Qi(\pi)=\sumiM\wi\Qi(\pi)=\Qw(\pi).
    \end{align*}
    Additionally,
    \begin{align*}
        \E_W\bigbra{\Deltatildew(\pia,\pib)}=\E_W\bigbra{\Qtildew(\pia)}-\E_W\bigbra{\Qtildew(\pib)}=\Qw(\pia)-\Qw(\pib)=\Deltaw(\pia,\pib).
    \end{align*}
    Therefore, we can follow a symmetrization argument to upper bound the expected oracle regret in terms of a Rademacher complexity, namely the weighted Rademacher complexity. Let $W'$ be an independent copy of $W$ and let $\eps=\{\epsij\mid c\in\calC,i\in[\nc]\}$ be a set of independent Rademacher random variables. Then,
    \begin{align*}
        &\E\bra{\supPiab|\Deltaw(\pia,\pib)-\Deltatildew(\pia,\pib)|} \\
        &=\E_W\bra{\supPiab\biggabs{\E_{W'}\biggbra{\sumijw \Delta(\Wijprime;\pia,\pib)}-\sumijw \Delta(\Wij;\pia,\pib)}} \\
        &=\E_W\bra{\supPiab\biggabs{\E_{W'}\biggbra{\sumijw \Delta(\Wijprime;\pia,\pib)-\sumijw \Delta(\Wij;\pia,\pib)}}} \\
        &\le\E_W\bra{\E_{W'}\bra{\supPiab\biggabs{\sumijw\Bigpar{\Delta(\Wijprime;\pia,\pib)-\Delta(\Wij;\pia,\pib)}}}} \\
        &=\E_{W,W',\eps}\bra{\supPiab\biggabs{\sumijw\epsij\Bigpar{\Delta(\Wijprime;\pia,\pib)-\Delta(\Wij;\pia,\pib)}}} \\
        &\le2\E_{W,\eps}\bra{\supPiab\biggabs{\sumijw\epsij \Delta(\Wij;\pia,\pib)}} \\
        &=2\frakR_{\w,\barn}(\Delta\calF_\Pi) \\
        &\le 4\frakR_{\w,\barn}(\calF_\Pi).
    \end{align*}
    The first equalities and inequalities follow from standard symmetrization arguments, and the last inequality follows from Lemma \ref{lem:RademacherComplexityDifferenceBound}.
    Next, we use this bound on the expectation of the oracle regret and Talagrand's inequality (Lemma \ref{lem:TalagrandInequality}), to establish a high-probability bound on the oracle regret.
    In particular, we identify the set of independent random variables $\tilde W=\{\tilde\Wij=(\Xij,\fracwini\Gij)\mid c\in\calC, i\in[\nc]\}$ and the function class $\calH=\{h(\cdot;\pia,\pib)\mid\pia,\pib\in\Pi\}$ where
    \begin{align}
        h(\tilde\Wij;\pia,\pib)=\E\regbra{\Delta(\tilde\Wij;\pia,\pib)}-\Delta(\tilde\Wij;\pia,\pib),
    \end{align}
    which is uniformly bounded for any $c\in\calC$ and $i\in[\nc]$ by
    \begin{align*}
        \regabs{h(\tilde\Wij;\pia,\pib)}&=\fracwini\Bigabs{\E\bigbra{\Gij(\pia(\Xij))-\Gij(\pib(\Xij))}-\bigpar{\Gij(\pia(\Xij))-\Gij(\pib(\Xij))}} \\
        &\le\fracwini 4U \\
        &\le4U\swnC\eqqcolon\UwnC,
    \end{align*}
    where $U>0$ is a uniform upper bound on $|\Gij(a)|$ for any $c\in\calC$ and $a\in\calA$ guaranteed by Assumption \ref{ass:dgp}, and where the last inequality follows from Inequality \eqref{eq:ClientDistributionRatioBound},
    Additionally, we have
    \begin{equation*}
        \swnC=\frac{1}{\sqrt{\min_{c\in\calC}n_c}}\sqrtfracskewnessn\le\littleo{\sqrtfracskewnessn},
    \end{equation*}
    as discussed in the proof of Proposition \ref{prop:WeightedRademacherBound}.
    Lastly, to use Talagrand's inequality, we set the constant $D$ (specified in Lemma \ref{lem:TalagrandInequality}) to be
    \begin{align*}
        D=\supPiab\sumij\E\bigbra{h^2(\tilde\Wij;\pia,\pib)} + 8\UwnC\E\bra{\supPiab\biggabs{\sumij\epsij h(\tilde\Wij;\pia,\pib)}}.
    \end{align*}
    By Lemma \ref{lem:ExpectedSupSumBound}, this choice of $D$ meets the required condition to use in Talagrand's inequality.
    In particular, we identify $\varphi:u\mapsto u^2$ defined over the set $\calU$ containing all possible outputs of any function in $\calH$ given any realization of $\tilde\Wij$ for any $c\in\calC$ as input.
    The uniform bound established above on realizable outputs of $h$ given input $\tilde\Wij$ implies that $\calU\subset[-\UwnC,\UwnC]$, and therefore, the Lipschitz constant of $\varphi$ is $L=2\UwnC$, as required.

    Next, after setting $t$ to be the positive solution of
    \begin{align*}
        \frac{t^2}{CD+C\UwnC t}=\log(C/\delta),
    \end{align*}
    Talagrand's inequality guarantees
    \begin{align*}
        &\P\rbra{\Bigabs{\supPiab\bigabs{\Deltaw(\pia,\pib)-\Deltatildew(\pia,\pib)}-\E\Bigbra{\supPiab\bigabs{\Deltaw(\pia,\pib)-\Deltatildew(\pia,\pib)}}}\ge t} \\
        &=\P\rbra{\biggabs{\supPi\Bigabs{\sumij h(\tilde\Wij;\pia,\pib)} - \E\bra{\supPi\biggabs{\sumij h(\tilde\Wij;\pia,\pib)} }} \ge t} \\
        &\le C\exp\rbra{-\frac{t}{C\UwnC}\log\rbra{1+\frac{\UwnC t}{D}}} \\
        &\le C\exp\rbra{-\frac{t^2}{CD+C\UwnC t}}=\delta.
    \end{align*}
    Here, we used the inequality $\log(1+x)\ge\frac{x}{1+x}$ for any $x\ge0$.
    Observe that, by construction,
    \begin{align*}
        t&=\frac{1}{2}C\UwnC\log(C/\delta)+\sqrt{\frac{1}{4}C^2\UwnC^2\log^2(C/\delta)+CD\log(C/\delta)} \\
        &\le C\UwnC\log(C/\delta) + \sqrt{CD\log(C/\delta)}
    \end{align*}
    and
    \begin{align*}
        D&=\supPiab\sumij\E\bigbra{h^2(\tilde\Wij;\pia,\pib)} + 8\Uwn\E\bra{\supPiab\biggabs{\sumij\epsij h(\tilde\Wij;\pia,\pib)}} \\
        &=\supPiab\sumij\fracwinisq\E\bra{\bigpar{\E\bra{\Delta(\Wij;\pia,\pib)}-\Delta(\Wij;\pia,\pib)}^2} \\
        &\hspace{1em} + 8\UwnC\E\bra{\supPiab\Biggabs{\sumij\fracwini\epsij\bigpar{\E\bra{\Delta(\Wij;\pia,\pib)}-\Delta(\Wij;\pia,\pib)}}} \\
        &=\supPiab\sumij\fracwinisq\Bigpar{\E\bra{\Delta^2(\Wij;\pia,\pib)}-\E\bra{\Delta(\Wij;\pia,\pib)}^2} \\
        &\hspace{1em} + 8\UwnC\E\bra{\supPiab\Biggabs{\sumij\fracwini\epsij\bigpar{\E\bra{\Delta(\Wij;\pia,\pib)}-\Delta(\Wij;\pia,\pib)}}} \\
        &\le \supPiab\sumij\fracwinisq\E\bra{\Delta^2(\Wij;\pia,\pib)} + 16\UwnC\E\bra{\supPiab\Biggabs{\sumij\fracwini\epsij\Delta(\Wij;\pia,\pib)}} \\
        &\le\supPiab\sumiM\fracwisqni\E\bra{\Delta^2(\Wij;\pia,\pib)} + 16\UwnC\frakRwn(\Delta\calF_\Pi) \\
        &\le\supPiab\sumiM\fracwisqni\E\bra{\Delta^2(\Wij;\pia,\pib)} + 32\UwnC\frakRwnF \\
        &=\fracVwn + 128U\frakRwnF\swnC.
    \end{align*}

    Therefore, with this setup, Talagrand's inequality guarantees that with probability at least $1-\delta$
    \begin{align*}
        &\supPiab\bigabs{\Deltaw(\pia,\pib)-\Deltatildew(\pia,\pib)} \\
        &\le\E\bra{\supPiab\bigabs{\Deltaw(\pia,\pib)-\Deltatildew(\pia,\pib)}} + t \\
        &=4\frakRwnF + \sqrt{CD\log(C/\delta)} + C\UwnC\log(C/\delta) \\
        &\le4\frakRwnF + \sqrt{C\rbra{\fracVwn + 128U\frakRwnF\swnC}\log\rbra{C/\delta}} + 4CU\swnC\log\rbra{C/\delta} \\
        &\le4\frakRwnF + \sqrt{C\log(C/\delta)\fracVwn } + \sqrt{128UC\log(C/\delta) \frakRwnF\swnC} + 4UC\log(C/\delta)\swnC\\
        &\le \rbra{\rbra{56+24\kappa(\Pi)}\sqrtfracVwn + \littleo{\sqrtfracskewnessn}} + \sqrt{C\log(C/\delta)\fracVwn} \\
        &\hspace{1.5em}+ \sqrt{\bigO{\sqrtfracskewnessn}\littleo{\sqrtfracskewnessn}} + \littleo{\sqrtfracskewnessn} \\
        &\le\rbra{56+24\kappa(\Pi) + \sqrt{C\log(C/\delta)}}\sqrtfracVwn + \littleo{\sqrtfracskewnessn} \\
        &\le\rbra{c_1\kappa(\Pi) + \sqrt{c_2\log(c_2/\delta)}}\sqrtfracVwn + \littleo{\sqrtfracskewnessn},
    \end{align*}
    where $c_1=24$ and $c_2$ is any constant such that $56+\sqrt{C\log(C/\delta)}\le\sqrt{c_2\log(c_2/\delta)}$.
    Here, we used the bounds previously established in the proof of Proposition \ref{prop:WeightedRademacherBound} that $\frakRwnF\le\calO\bigpar{\sqrt{\skewness/n}\,}$ and $\swnC\le o\bigpar{\sqrt{\skewness/n}\,}$.
    
\end{proof}

\subsection{Bounding Approximate Regret}\label{app:BoundingApproximateRegret}

\begin{proposition}\label{prop:ApproximateRegretBound}
Suppose Assumptions \ref{ass:dgp}, \ref{ass:LocalDataSizeScaling}, and \ref{ass:FiniteSampleError} hold. Then,
\begin{equation*}
    \supPiab|\Deltatildew(\pia,\pib)-\Deltahatw(\pia,\pib)|\le \littleop{\sqrtfracskewnessn}
\end{equation*}
\end{proposition}
\begin{proof}
    Recall that $\{(\Xij,\Aij,\Yij)\}_{i=1}^{\nc}$ is the data collected by client $c\in\calC$ as described in Section \ref{sec:Preliminaries-Data}.
    We assume each client estimates the local nuisance parameters using a cross-fitting strategy, as discussed in Algorithm \ref{alg:CAIPW}. Under this strategy, each client $c\in\calC$ divides their local dataset into $K$ folds, and for each fold $k$, the client estimates $\mui$ and $\oi$ using the rest $K-1$ folds. Let $k_c:[\nc]\to[K]$ denote the surjective mapping that maps a data point index to its corresponding fold containing the data point. We let $\muhatikij$ and $\ohatikij$ denote the estimators of $\mui$ and $\ei$ fitted on the $K-1$ folds of client $c$ other than $\kij$.

    As discussed Section \ref{sec:Approach}, recall the oracle AIPW scores
    \begin{equation*}
        \Gij(a)=\mu(\Xij;a)+\bigpar{\Yij-\mu(\Xij;a)}\oi(\Xij;a)\ones\{\Aij=a\}
    \end{equation*}
    and approximate AIPW scores
    \begin{equation*}
        \Ghatij(a)=\muhatikij(\Xij;a)+\bigpar{\Yij-\muhatikij(\Xij;a)}\ohatikij(\Xij;a)\ones\{\Aij=a\}
    \end{equation*}
    for any $a\in\calA$, where $\kij$ is the fold corresponding to data point $i$ of client $c$. One can verify that the difference between the oracle and approximate AIPW scores can be expressed as
    \begin{align*}
        \Ghatij(a)-\Gij(a)=\Gijp(a)+\Gijpp(a)+\Gijppp(a),
    \end{align*}
    where
    \begin{align*}
        \Gijp(a)&=\rbra{\muhatikij(\Xij;a)-\mui(\Xij;a)}\bigpar{1-\oi(\Xij;a)\ones\{\Aij=a\}}, \\
        \Gijpp(a)&=\bigpar{\Yij(a)-\mui(\Xij;a)}\rbra{\ohatikij(\Xij;a)-\oi(\Xij;a)}\ones\{\Aij=a\}, \\
        \Gijppp(a)&=\rbra{\mui(\Xij;a)-\muhatikij(\Xij;a)}\rbra{\ohatikij(\Xij;a)-\oi(\Xij;a)}\ones\{\Aij=a\}.
    \end{align*}

    This induces the following decomposition of the approximate regret:
    \begin{align*}
        \Deltahatw(\pia,\pib)-\Deltatildew(\pia,\pib)=S_1(\pia,\pib)+S_2(\pia,\pib)+S_3(\pia,\pib),
    \end{align*}
    where
    \begin{align*}
        S_1(\pia,\pib)&=\sumijw\Gijp(\pia(\Xij))-\Gijp(\pib(\Xij)) ,\\
        S_2(\pia,\pib)&=\sumijw\Gijpp(\pia(\Xij))-\Gijpp(\pib(\Xij)), \\
        S_3(\pia,\pib)&=\sumijw\Gijppp(\pia(\Xij))-\Gijppp(\pib(\Xij)).
    \end{align*}
    We further decompose $S_1$ and $S_2$ by folds as follows:
    \begin{align*}
        S_1(\pia,\pib)&=\sumkK S_1^k(\pia,\pib), \\
        S_2(\pia,\pib)&=\sumkK S_2^k(\pia,\pib),
    \end{align*}
    where
    \begin{align*}
        S_1^k(\pia,\pib)=\sumiM\fracwini\sum_{\{i|\kij=k\}}\Gijp(\pia(\Xij))-\Gijp(\pib(\Xij)), \\
        S_2^k(\pia,\pib)=\sumiM\fracwini\sum_{\{i|\kij=k\}}\Gijpp(\pia(\Xij))-\Gijpp(\pib(\Xij)),
    \end{align*}
    for each $k\in[K]$. To determine a bound on the approximate regret, we will establish high probability bounds for the worst-case absolute value over policies of each term in this decomposition. For convenience, for any policy $\pi$, we will denote $\pi(x;a)=\ones\{\pi(x)=a\}$.

    \vspace{.5em}
    \textit{\underline{Bounding $S_1$}}: We wish to bound $\supPiab\abs{S_1(\pia,\pib)}$. We first bound $\supPiab\abs{S_1^k(\pia,\pib)}$ for any $k\in[K]$.
    
    First, note that since $\muhatikij$ is estimated using data outside fold $\kij$, when we condition on the data outside fold $\kij$, $\muhatikij$ is fixed and each term in $S_1(\pia,\pib)$ is independent. This allows us to compute
    \begin{align*}
        &\E\bra{\Gijp(\pia(\Xij))-\Gij(\pib(\Xij))} \\
        &=\sumaA\E\bra{\bigpar{\pia(\Xij;a)-\pib(\Xij;a)}\rbra{\muhatikij(\Xij;a)-\mui(\Xij;a)}\bigpar{1-\oi(\Xij;a)\ones\{\Aij=a\}}} \\
        &=\sumaA\E\bra{\E\bra{\bigpar{\pia(\Xij;a)-\pib(\Xij;a)}\rbra{\muhatikij(\Xij;a)-\mui(\Xij;a)}\bigpar{1-\oi(\Xij;a)\ones\{\Aij=a\}}\Bigmid \Xij}} \\
        &=\sumaA\E\bra{\bigpar{\pia(\Xij;a)-\pib(\Xij;a)}\rbra{\muhatikij(\Xij;a)-\mui(\Xij;a)}\E\bra{1-\oi(\Xij;a)\ones\{\Aij=a\}\Bigmid \Xij}}=0
    \end{align*}
    Therefore,
    \begin{align*}
        &K\supPiab\abs{S_1^k(\pia,\pib)} \\
        &\le\supPiab\Biggabs{\sumiM\frac{\wi}{\nc/K}\sum_{\{i|\kij=k\}}\Gijp(\pia(\Xij))-\Gijp(\pib(\Xij))} \\
        &=\supPiab\Biggabs{\sumiM\frac{\wi}{\nc/K}\sum_{\{i|\kij=k\}}\bigpar{\Gijp(\pia(\Xij))-\Gijp(\pib(\Xij))}-\E\bra{\Gijp(\pia(\Xij))-\Gijp(\pib(\Xij))}}.
    \end{align*}
    Identifying $\Gijp$ with $\Gij$ and sample sizes $\nC/K$ with $\nC$, the right-hand side in the above inequality is effectively an oracle regret and so we can apply Proposition \ref{prop:OracleRegretBound} to obtain that with probability at least $1-\delta$,
    \begin{align*}
        &K\supPiab\abs{S_1^k(\pia,\pib)} \\
        &\le\supPiab\biggabs{\sumiM\frac{\wi}{\nc/K}\sum_{\{j|k_i(j)=k\}}\bigpar{\Gijp(\pia(\Xij))-\Gijp(\pib(\Xij))}-\E\bra{\Gijp(\pia(\Xij))-\Gijp(\pib(\Xij))}} \\
        &\le C_{\Pi,\delta}\sqrt{\frac{\supPiab\sumiM \fracwisqbarni\E\bra{\bigpar{\Gijp(\pia(\Xij))-\Gijp(\pib(\Xij))}^2\mid\muhatikij}}{n/K}} + \littleo{\sqrt{\frac{\skewness}{n/K}}} \\
        &\le C_{\Pi,\delta}\sqrt{K\supPiab\sumiM \fracwisqni\E\bra{\bigpar{\Gijp(\pia(\Xij))-\Gijp(\pib(\Xij))}^2\mid\muhatikij}} + \littleo{\sqrtfracskewnessn} \\
        &\le C_{\Pi,\delta}\rbra{1/\eta-1}\sqrt{2K\sumiM\fracwisqni\E\bra{||\muhatikij(\Xij)-\mui(\Xij)||_2^2\mid\muhatikij}}+ \littleo{\sqrtfracskewnessn},
    \end{align*}
    where $C_{\Pi,\delta}=c_1\kappa(\Pi) + \sqrt{c_2\log(c_2/\delta)}$ for some universal constants $c_1$ and $c_2$, and $\eta=\min_{c\in\calC}\eta_c$ for $\eta_c$ in the overlap assumption stated in in Assumption \ref{ass:dgp}. The last inequality follows from a uniform bound on $\Gijp(\pia(\Xij))-\Gij(\pib(\Xij))$ and the overlap assumption.
    
    By the assumption on finite sample error bounds for the nuisance functions stated in Assumption \ref{ass:FiniteSampleError}, for every $c\in\calC$
    \begin{align*}
        \E\bra{||\muhatikij(\Xij)-\mui(\Xij)||^2\mid\muhatikij}\le\frac{g_c\rbra{\alpha_K\nc}}{\rbra{\alpha_K\nc}^{\zeta_\mu}},
    \end{align*}
    where $\alpha_K=1-K^{-1}$, $g_c$ is some decreasing function, and $0<\zeta_\mu<1$. Then,
    \begin{align*}
        &\sumiM\fracwisqni\E\bra{||\muhatikij(\Xij)-\mui(\Xij)||^2\mid\muhatikij} \\
        &\le\sumiM\fracwisqni\frac{g_c(\alphaK\nc)}{(\alphaK\nc)^\zetamu} \\
        &\le\frac{\max_{c\in\calC} g_c(
        \alphaK\nc)}{\alphaK^\zetamu\cdot\min_{c\in\calC}\nc^\zetamu}\sumiM\frac{\wi^2}{\nc} \\
        &\le\frac{\max_{c\in\calC} g_c(
        \alphaK\nc)}{\alphaK^\zetamu\cdot\min_{c\in\calC}\nc^\zetamu}\fracskewnessn.
    \end{align*}
    By the local data size scaling assumption in Assumption \ref{ass:LocalDataSizeScaling}, for any $c\in\calC$, we have that $\nc=\Omega(\nu_c(n))$ where $\nu_c$ is an increasing function.
    In other words, there exists a constant $\tau>0$ such that $\nc\ge \tau\nuc(n)$ for sufficiently large $n$. Then, since $\gc$ is decreasing, $\gc(\alphaK n_c)< \gc(\tau\alphaK\nuc(n))$ for sufficiently large $n$.
    Moreover, since $\nuc$ is increasing and $\tau\alpha_K>0$, $\tilde\nu_c=\tau\alphaK\nuc$ is also increasing, and since $\gc$ is decreasing, the composition $\tilde g_c=\gc\circ\tilde\nu_c$ is decreasing. Therefore, $\gc(\alphaK\nc)$ is asymptotically bounded by a decreasing function $\tilde g_c$ of $n$.
    This observation and the fact that the maximum of a set of decreasing functions is itself decreasing imply that $\max_{c\in\calC}\gc(\alphaK\nc)$ is asymptotically bounded by the decreasing function $\tilde g$ defined by $\tilde g(n)=\max_{c\in\calC}\tilde g_c(n)$.
    In other words,
    \begin{align*}
        \max_{c\in\calC}\gc(\alphaK\nc)\le\tilde g(n)\le o(1).
    \end{align*}
    Additionally, since $\nc=\Omega(\nuc(n))$ and $\zetamu>0$, we also have that
    \begin{align*}
        \frac{1}{\min_{c\in\calC}\nc^\zetamu}\le o(1).
    \end{align*}
    These two observations imply
    \begin{align*}
        \sumiM\fracwisqni\E\bra{||\muhatikij(\Xij)-\mui(\Xij)||^2\mid\muhatikij} &\le\frac{\max_{c\in\calC} g_c(
        \alphaK\nc)}{\alphaK^\zetamu\cdot\min_{c\in\calC}\nc^\zetamu}\fracskewnessn\le o\rbra{\fracskewnessn}.
    \end{align*}
    Therefore,
    \begin{align*}
        &\supPiab\abs{S_1^k(\pia,\pib)} \\
        &\le C_{\Pi,\delta}\rbra{1/\eta-1}\sqrt{\frac{2}{K}\sumiM\fracwisqni\E\bra{||\muhatikij(\Xij)-\mui(\Xij)||_2^2\mid\muhatikij}}+ \littleo{\sqrtfracskewnessn} \\
        &\le C_{\Pi,\delta}\rbra{1/\eta-1}\sqrt{\frac{2}{K}\cdot o\rbra{\fracskewnessn}}+ \littleo{\sqrtfracskewnessn}\le\littleo{\sqrtfracskewnessn},
    \end{align*}
    and
    \begin{align*}
        \supPiab\abs{S_1(\pia,\pib)}&\le\sumkK\supPiab\abs{S_1^k(\pia,\pib)}\le\littleo{\sqrtfracskewnessn}.
    \end{align*}

    \vspace{.5em}
    \textit{\underline{Bounding $S_2$}}: The bound for $\supPiab\abs{S_2(\pia,\pib)}$ follows the same argument as that of $S_1$. We first bound $\supPiab\abs{S_2^k(\pia,\pib)}$ for any $k\in[K]$.
    
    First, note that since $\ohatikij$ is estimated using data outside fold $\kij$, when we condition on the data outside fold $\kij$, $\ohatikij$ is fixed and each term in $S_2(\pia,\pib)$ is independent. This allows us to compute
    \begin{align*}
        &\E\bra{\Gijpp(\pia(\Xij))-\Gijpp(\pib(\Xij))} \\
        &=\E\bra{\sumaA\rbra{\pia(\Xij;a)-\pib(\Xij;a)}\rbra{\Yij(a)-\mui(\Xij;a)}\rbra{\ohatikij(\Xij;a)-\oi(\Xij;a)}\ones\{\Aij=a\}} \\
        &=\E\bra{\rbra{\pia(\Xij;\Aij)-\pib(\Xij;\Aij)}\rbra{\Yij(\Aij)-\mui(\Xij;\Aij)}\rbra{\ohatikij(\Xij;a)-\oi(\Xij;a)}} \\
        &=\E\bra{\E\bra{\rbra{\pia(\Xij;\Aij)-\pib(\Xij;\Aij)}\rbra{\Yij(\Aij)-\mui(\Xij;\Aij)}\rbra{\ohatikij(\Xij;a)-\oi(\Xij;a)}\Bigmid \Xij,\Aij}} \\
        &=\E\bra{\rbra{\pia(\Xij;\Aij)-\pib(\Xij;\Aij)}\E\Bigbra{\Yij(\Aij)-\mui(\Xij;\Aij)\mid\Xij,\Aij}\rbra{\ohatikij(\Xij;a)-\oi(\Xij;a)}}=0
    \end{align*}
    Therefore, we can follow the exact same argument as above, eliciting Proposition \ref{prop:OracleRegretBound}, to obtain that with probability at least $1-\delta$,
    \begin{align*}
        &K\supPiab\abs{S_2^k(\pia,\pib)} \\
        &\le\supPiab\Biggabs{\sumiM\frac{\wi}{\nc/K}\sum_{\{i|\kij=k\}}\Gijpp(\pia(\Xij))-\Gijpp(\pib(\Xij))} \\
        &=\supPiab\Biggabs{\sumiM\frac{\wi}{\nc/K}\sum_{\{i|\kij=k\}}\bigpar{\Gijpp(\pia(\Xij))-\Gijpp(\pib(\Xij))}-\E\bra{\Gijpp(\pia(\Xij))-\Gijpp(\pib(\Xij))}} \\
        &\le\supPiab\biggabs{\sumiM\frac{\wi}{\nc/K}\sum_{\{j|k_i(j)=k\}}\bigpar{\Gijpp(\pia(\Xij))-\Gijpp(\pib(\Xij))}-\E\bra{\Gijpp(\pia(\Xij))-\Gijpp(\pib(\Xij))}} \\
        &\le C_{\Pi,\delta}\sqrt{\frac{\supPiab\sumiM \fracwisqbarni\E\bra{\bigpar{\Gijpp(\pia(\Xij))-\Gijpp(\pib(\Xij))}^2\mid\ohatikij}}{n/K}} + \littleo{\sqrt{\frac{\skewness}{n/K}}} \\
        &\le C_{\Pi,\delta}\sqrt{K\supPiab\sumiM \fracwisqni\E\bra{\bigpar{\Gijpp(\pia(\Xij))-\Gijpp(\pib(\Xij))}^2\mid\ohatikij}} + \littleo{\sqrtfracskewnessn} \\
        &\le C_{\Pi,\delta}\sqrt{4BK\sumiM\fracwisqni\E\bra{||\ohatikij(\Xij)-\oi(\Xij)||_2^2\mid\ohatikij}}+ \littleo{\sqrtfracskewnessn},
    \end{align*}
    where $C_{\Pi,\delta}=c_1\kappa(\Pi) + \sqrt{c_2\log(c_2/\delta)}$ for some universal constants $c_1$ and $c_2$, and $B=\maxC B_c$ for the bounds $B_c$ on the outcomes defined in Assumption \ref{ass:dgp}.
    The last inequality follows from a uniform bound on $\Gijpp(\pia(\Xij))-\Gijpp(\pib(\Xij))$.

    We follow the exact same argument as above to get
    \begin{align*}
        \sumiM\fracwisqni\E\bra{||\ohatikij(\Xij)-\oi(\Xij)||^2\mid\ohatikij} &\le o\rbra{\fracskewnessn}.
    \end{align*}

    Therefore,
    \begin{align*}
        &\supPiab\abs{S_2^k(\pia,\pib)} \\
        &\le C_{\Pi,\delta}\sqrt{\frac{4B}{K}\sumiM\fracwisqni\E\bra{||\ohatikij(\Xij)-\oi(\Xij)||_2^2\mid\ohatikij}}+ \littleo{\sqrtfracskewnessn} \\
        &\le C_{\Pi,\delta}\sqrt{\frac{4B}{K}\cdot o\rbra{\fracskewnessn}}+ \littleo{\sqrtfracskewnessn} \\
        &\le\littleo{\sqrtfracskewnessn},
    \end{align*}
    and
    \begin{align*}
        \supPiab\abs{S_2(\pia,\pib)}&\le\sumkK\supPiab\abs{S_2^k(\pia,\pib)}\le\littleo{\sqrtfracskewnessn}.
    \end{align*}

    \vspace{.5em}
    \textit{\underline{Bounding $S_3$}}: Next, we bound the contribution from $S_3$. We have that
    \begin{align*}
        &\supPiab\abs{S_3(\pia,\pib)} \\
        &=\supPiab\abs{\sumijw\Gijppp(\pia(\Xij))-\Gijppp(\pib(\Xij))} \\
        &\le2\abs{\sumijw\sum_{a\in\calA}\rbra{\mui(\Xij;a)-\muhatikij(\Xij;a)}\rbra{\ohatikij(\Xij;a)-\oi(\Xij;a)}} \\
        &\le2\sqrt{\sumijw\bignorm{\mui(\Xij)-\muhatikij(\Xij)}_2^2}\sqrt{\sumijw\bignorm{\ohatikij(\Xij)-\oi(\Xij)}_2^2} \\
        &\le2\sqrt{\sumiM\wi\frac{\gc(\alphaK\nc)}{(\alphaK\nc)^\zetamu}}\sqrt{\sumiM\wi\frac{\gc(\alphaK\nc)}{(\alphaK\nc)^\zetao}} \\
        &\le\frac{2}{\alphaK^{(\zetamu+\zetao)/2}}\sqrt{\maxC\frac{\wi}{\nc^\zetamu}\sumiM\gc(\alphaK\nc)}\sqrt{\maxC\frac{\wi}{\nc^\zetao}\sumiM\gc(\alphaK\nc)} \\
        &=\frac{2}{\alphaK^{(\zetamu+\zetao)/2}}\sumiM\gc(\alphaK\nc)\sqrt{\maxC\frac{\wi^2}{\nc^{\zetamu+\zetao}}} \\
        &\le\frac{2}{\alphaK^{(\zetamu+\zetao)/2}}\sumiM\gc(\alphaK\nc)\sqrt{\maxC\frac{\wi^2}{\nc}} \\
        &\le\frac{2}{\alphaK^{(\zetamu+\zetao)/2}}\sumiM\gc(\alphaK\nc)\sqrt{\sumiM\frac{\wi^2}{\nc}} \\
        &\le\frac{2}{\alphaK^{(\zetamu+\zetao)/2}}\sumiM\gc(\alphaK\nc)\sqrtfracskewnessn.
    \end{align*}
    As discussed earlier, $\gc(\alphaK\nc)$ is asymptotically bounded by a decreasing function of $n$. Since the sum of decreasing functions is decreasing, $\sumiM\gc(\alphaK\nc)$ is asymptotically bounded by a decreasing function $\tilde g$ in $n$. In other words, $\sumiM\gc(\alphaK\nc)\le\tilde g(n)\le o(1)$.
    Therefore,
    \begin{align*}
        \supPiab\abs{S_3(\pia,\pib)}\le\frac{2}{\alphaK^{(\zetamu+\zetao)/2}}\cdot o(1)\cdot\sqrtfracskewnessn\le\littleo{\sqrtfracskewnessn}.
    \end{align*}

    \vspace{.5em}
    Putting all the above bounds together, we have
    \begin{align*}
        \supPiab|\Deltatildew(\pia,\pib)-\Deltahatw(\pia,\pib)|&\le\supPiab\abs{S_1(\pia,\pib)+S_2(\pia,\pib)+S_3(\pia,\pib)} \\
        &\le\supPiab\abs{S_1(\pia,\pib)}+\supPiab\abs{S_2(\pia,\pib)}+\supPiab\abs{S_3(\pia,\pib)} \\
        &\le\littleo{\sqrtfracskewnessn}.
    \end{align*}
\end{proof}

\newpage
\subsection{Proof of Theorem \ref{thm:MainTheorem}}

\MainTheorem*

\begin{proof}
    Let $\pistarw=\argmax_{\pi\in\Pi}\Qw(\pi)$.
    Using the results of Propositions \ref{prop:OracleRegretBound} and \ref{prop:ApproximateRegretBound}, with probability at least $1-\delta$, we have
    \begin{align*}
        R_\w(\pihatw)&=\Qw(\pistarw)-\Qw(\pihatw) \\
        &=\bigpar{\Qw(\pistarw)-\Qw(\pihatw)} - \bigpar{\Qhatw(\pistarw)-\Qhatw(\pihatw)} + \bigpar{\Qhatw(\pistarw)-\Qhatw(\pihatw)} \\
        &=\Deltaw(\pistarw,\pihatw) - \Deltahatw(\pistarw,\pihatw) + \bigpar{\Qhatw(\pistarw)-\Qhatw(\pihatw)} \\
        &\le\Deltaw(\pistarw,\pihatw) - \Deltahatw(\pistarw,\pihatw) \\
        &\le\supPiab\regabs{\Deltaw(\pia,\pib)-\Deltahatw(\pia,\pib)} \\
        &\le\supPiab\regabs{\Deltaw(\pia,\pib)-\Deltatildew(\pia,\pib)}+\supPiab\regabs{\Deltatildew(\pia,\pib)-\Deltahatw(\pia,\pib)} \\
        &\le\rbra{\rbra{c_1\kappa(\Pi) + \sqrt{c_2\log(c_2/\delta)}}\sqrtfracVwn + \littleo{\sqrtfracskewnessn}} + \littleop{\sqrtfracskewnessn} \\
        &\le\rbra{c_1\kappa(\Pi) + \sqrt{c_2\log(c_2/\delta)}}\sqrtfracVwn + \littleop{\sqrtfracskewnessn},
    \end{align*}
    where $c_1$ and $c_2$ are universal constants.
    Lastly, we decompose the weighted variance term by
    \begin{align*}
        \Vwn&=\supPiab\sumiM\fracwisqbarni\opsE_{\Zi\sim\calDcec}\bigbra{\bigpar{\Gi(\pia(\Xi))-\Gi(\pib(\Xi))}^2} \\
        &\le\maxC\supPiab\opsE_{\Zi\sim\calDcec}\bigbra{\bigpar{\Gi(\pia(\Xi))-\Gi(\pib(\Xi))}^2}\cdot\sumiM\fracwisqbarni \\
        &\le4\cdot\maxC\supPi\opsE_{\Zi\sim\calDcec}\regbra{{\Gi(\pi(\Xi))}}\cdot\sumiM\fracwisqbarni \\
        &=4V\cdot\skewness.
    \end{align*}
    We absorb the factor of $\sqrt{4}$ into the universal constants to get the desired result.
\end{proof}

\section{Bounding Local Regret}\label{app:BoundingLocalRegret}

\subsection{Proof of Theorem \ref{thm:SubMainTheorem}}

\SubMainTheorem*

\begin{proof}
    Let $\pistari=\argmax_{\pi\in\Pi}\Qi(\pi)$.
    Then,
    \begin{align*}
        \Ri(\pihatw)&=\Qi(\pistari)-\Qi(\pihatw) \\
        &=\Qi(\pistari)-\Qi(\pihatw) \mp \Qw(\pistari) \pm \Qw(\pihatw) \\
        &=\bigpar{\Qi(\pistari)-\Qw(\pistari)} + \bigpar{\Qw(\pihatw)-\Qi(\pihatw)} + \bigpar{\Qw(\pistari)-\Qw(\pihatw)} \\
        &\le2\supPi\abs{\Qi(\pi)-\Qw(\pi)} + \bigpar{\Qw(\pistari)-\Qw(\pihatw)} \\
        &\le2\supPi\abs{\Qi(\pi)-\Qw(\pi)} + \bigpar{\Qw(\pistarw)-\Qw(\pihatw)} \\
        &=2\supPi\abs{\Qi(\pi)-\Qw(\pi)} + \Rw(\pihatw).
    \end{align*}

    By Lemma \ref{lem:ExpectedOracleEqualsLocalPolicyValue}, we can express the the local policy value as
    \begin{align*}
        \Qi(\pi)=\opsE_{Z\sim\calDcec}\bra{\Gamma(\pi(X))}
    \end{align*}
    where $(\Gamma(a_1),\dots,\Gamma(a_d))$ are the constructed AIPW scores from a context-action-outcomes sample $Z=(X,A,Y(a_1),\dots,Y(a_d))\sim\calDcec$.
    In addition, the global policy value can be expressed as
    \begin{align*}
        \Qw(\pi)=\opsE_{Z\sim\calDwew}\bra{\Gamma(\pi(X))}.
    \end{align*}
    where $(\Gamma(a_1),\dots,\Gamma(a_d))$ are the constructed AIPW scores from a context-action-outcomes sample $Z=(X,A,Y(a_1),\dots,Y(a_d))$ such that $c\sim\w$ and then $Z\sim\calDcec$. 
    Therefore,
    \begin{align}\label{eq:IrreducibleRegretTensorization-eq1}
        \supPi\abs{\Qi(\pi)-\Qw(\pi)}&=\supPi\Bigabs{\opsE_{Z\sim\calDcec}\regbra{\Gamma(\pi(X))}-\opsE_{Z\sim\calDwew}\regbra{\Gamma(\pi(X))}}
    \end{align}
    By the boundedness and overlap assumption in Assumption \ref{ass:dgp}, one can easily verify the uniform bound
    \begin{equation*}
        \abs{\Gi(a)}\le B_c + 2B_c/\eta_c\le 3B_c/\eta_c\le 3B/\eta\eqqcolon U
    \end{equation*}
    for any constructed AIPW score $\Gi(a)$ for any $a\in\calA$ and any client $c\in\calC$.
    Therefore, Equation \eqref{eq:IrreducibleRegretTensorization-eq1} is bounded by the the integral probability metric distance \citep{sriperumbudur2009integral} between $\calDcec$ and $\calDwew$ under uniformly bounded test functions since
    \begin{align*}
        \{Q(T(\cdot);\pi)\mid\pi\in\Pi\}\subset\calF_{\infty}^U\coloneqq\{f\mid\regnorm{f}_\infty\le U\},
    \end{align*}
    where $T(X,A,Y(a_1),\dots,Y(a_d))=(X,\Gamma(a_1),\dots,\Gamma(a_d))$.
    Thus,
    \begin{align*}
        \supPi\abs{\Qi(\pi)-\Qw(\pi)}&=\supPi\Bigabs{\opsE_{Z\sim\calDcec}\bigbra{Q(T(Z);\pi)}-\opsE_{Z\sim\calDwew}\bigbra{Q(T(Z);\pi)}} \\
        &\le \sup_{f\in\calF_{\infty}^U}\Bigabs{\opsE_{Z\sim\calDcec}\bigbra{f(Z)}-\opsE_{Z\sim\calDwew}\bigbra{f(Z)}} \\
        &=U\cdot \TV(\calDcec,\calDwew).
    \end{align*}
    The last equality holds by the definition of the total variation distance as an integral probability metric with uniformly bounded test functions.
\end{proof}

\subsection{Distribution Shift Bound}\label{app:DistributionShiftBound}

First, we state some important properties of the KL divergence.

\begin{lemma}
    The KL divergence has the following properties.
    \begin{itemize}
        \item Tensorization Property:
        Let $\calP=\prod_{i=1}^m\calP_i$ and $\calQ=\prod_{i=1}^m\calQ_i$ be two product distributions. Then,
        \begin{equation*}
            \KL(\calP||\calQ)=\sum_{i=1}^m\KL(\calP_i||\calQ_i).
        \end{equation*}
    
        \item Chain Rule: Let $\calP_{XY}=\calP_X\calP_{Y|X}$ and $\calQ_{XY}=\calQ_X\calQ_{Y|X}$ be two distributions for a pair of random variables $X, Y$. Then,
        \begin{equation*}
            \KL(\calP_{XY}||\calQ_{XY})=\KL(\calP_X||\calQ_X)+\KL(\calP_{Y|X}||\calQ_{Y|X}\mid\calP_X)
        \end{equation*}
        where
        \begin{equation*}
            \KL(\calP_{Y|X}||\calQ_{Y|X}\mid\calP_X)=\opsE_{X\sim\calP_X}\bra{\KL(\calP_{Y|X}||\calQ_{Y|X})}.
        \end{equation*}
    \end{itemize}
\end{lemma}

Then, we can use these properties to additively separate the sources of distribution shift in our local regret bound.

\SubSubMainTheorem*

\begin{proof}
    For any $c\in\calC$, the joint probability density function of $\calDcec$ factorizes as
    \begin{align*}
        p_{\Xi,\Ai,\vecYi}(x,a,y)=p_{\Xi}(x)\ei(a|x)p_{\vec{Y}^c|\Xi,\Ai}(y|x,a)=p_{\Xi}(x)\ei(a|x)p_{\vec{Y}^c|\Xi}(y|x)
    \end{align*}
    for any $(x,a,y)\in\calX\times\calA\times\calY^d$, where the last equality holds by the unconfoundedness property stated in Assumption \ref{ass:dgp}.
    Next, let $\Sigma$ be the $\sigma$-field over $\calX\times\calA\times\calY^d$ on which the $\calDcec$ are defined.
    We have that
    \begin{align*}
        \TV(\calDcec,\calDwew)&=\sup_{A\subset\Sigma}\bigabs{\calDcec(A)-\calDwew(A)} \\
        &=\sup_{A\in\Sigma}\bigabs{\calDcec(A)-\sumiM\wi\calDkek(A)} \\
        &=\sup_{A\in\Sigma}\bigabs{\sumiM\wi\rbra{\calDcec(A)-\calDkek(A)}} \\
        &\le\sup_{A\in\Sigma}\sumiM\wi\bigabs{\calDcec(A)-\calDkek(A)} \\
        &\le\sumiM\wi\sup_{A\in\Sigma}\bigabs{\calDcec(A)-\calDkek(A)} \\
        &=\opsE_{k\sim\w}\bra{\TV(\calDcec,\calDkek)} \\
        &\le\opsE_{k\sim\w}\bra{\sqrt{\KL(\calDcec||\calDkek)}\,},
    \end{align*}
    where the last inequality holds by Pinsker's inequality.
    Moreover,
    \begin{align*}
        \KL(\calDcec||\calDkek) &=\KL(p_{\Xi,\Ai,\vecYi}||p_{X_k,A_k,\vec{Y}_k}) \\
        &=\KL(p_{\Xi}||p_{\Xk}) + \KL(\ei p_{\vec{Y}^c|\Xi}||\ek p_{\vec{Y}^k|\Xk}\mid p_{\Xi}) \\
        &=\KL\bigpar{p_{\Xi}||p_{\Xk}} + \KL\bigpar{\ei||\ek\mid p_{\Xi}} +  \KL\bigpar{p_{\vec{Y}^c|\Xi}||p_{\vec{Y}^k|\Xk}\mid p_{\Xi}} \\
        &=\KL\bigpar{p_{\Xi}||p_{\Xk}} + \KL\bigpar{\ei||\ek} +  \KL\bigpar{p_{\vec{Y}^c|\Xi}||p_{\vec{Y}^k|\Xk}},
    \end{align*}
    where the first equality holds by the chain rule of the KL divergence and the second equality holds by the tensorization property of KL divergence. In the last inequality, for the sake of brevity, we just get rid of the explicit marker representing conditional KL divergence. It is understood that when the distributions are conditional distributions, their KL divergence is a conditional KL divergence.
\end{proof}

\subsection{Alternative Local Regret Bound}\label{app:AlternativeLocalRegretBound}

We provide an alternative local regret bound that is applicable in scenarios where the AIPW score variance is significantly less than the AIPW score range.

\begin{theorem}
    Suppose Assumption \ref{ass:dgp} holds. Then, for any $c\in\calC$,
    \begin{equation*}
        R_c(\pihatw)\le\sqrt{\smash[b]{4V\cdot\chi^2(\calDcec||\calDwew)}} + R_\w(\pihatw),
    \end{equation*}
    where $V=\maxC\supPi\E_{\calDcec}[\Gi(\pi(\Xi))]$.
\end{theorem}
\begin{proof}
    As shown in Theorem \ref{thm:SubMainTheorem},
    \begin{align*}
        R_c(\pihatw)\le2\supPi\regabs{\Qi(\pi)-\Qw(\pi)} + R_\w(\pihatw).
    \end{align*}
    Thus, we seek a new bound on the first term due to distribution shift.

    Let $p_c(z)$ and $p_\w(z)$ for any $z\in\calX\times\calA\times\calY^d$ be the joint probability density functions of $\calDcec$ and $\calDwew$, respectively. Additionally, for any $c\in\calC$ and any $Z\sim\calDcec$, let $f(Z;\pi)=\Gi(\pi(\Xi))$. Then, we can do the following calculations to get
    \begin{align*}
        \supPi\abs{\Qi(\pi)-\Qw(\pi)}&=\supPi\abs{\opsE_{\Zi\sim\calDcec}[\Gi(\pi(\Xi))]-\opsE_{c\sim\w}\opsE_{\Zi\sim\calDcec}[\Gi(\pi(\Xi))]} \\
        &=\supPi\abs{\opsE_{Z\sim\calDcec}[f(Z;\pi)]-\opsE_{Z\sim\calDwew}[f(Z;\pi)]} \\
        &=\supPi\abs{\int f(z;\pi)p_c(z)dz-\int f(z;\pi)p_\w(z)dz} \\
        &=\supPi\abs{\int f(z;\pi)\sqrt{p_\w(z)}\rbra{\frac{p_c(z)-p_\w(z)}{\sqrt{p_\w(z)}}}dz} \\
        &\le\supPi\sqrt{\int f(z;\pi)^2p_\w(z)dz\cdot\int\frac{(p_c(z)-p_\w(z))^2}{p_\w(z)}dz} \\
        &=\supPi\sqrt{\opsE_{Z\sim\calDwew}[f(Z;\pi)^2]\cdot\chi^2(\calDcec||\calDwew)} \\
        &\le\supPi\sqrt{\max_{c\in\calC}\opsE_{Z\sim\calDcec}[\Gi(\pi(\Xi))^2]\cdot\chi^2(\calDcec||\calDwew)} \\
        &=\sqrt{V\cdot\chi^2(\calDcec||\calDwew)}.
    \end{align*}
    Thus, we get the desired result.
\end{proof}

Compare the distribution shift term in this alternate result to the distribution shift term $U\cdot\TV(\calDcec,\calDwew)$ in the local regret bound we established in Theorem \ref{thm:SubMainTheorem}. The alternate bound is useful in that it does not rely on bounded AIPW scores, and instead is scaled by the maximum variance of the AIPW scores, which may be smaller than the range and it also appears in our global regret bound. Therefore, it is a more natural bound in this sense. However, the chi-squared divergence does not have a chain rule that would allow us to additively separate the sources of distribution shift in this bound, as we did in Proposition \ref{thm:SubSubMainTheorem}. The reason for this limitation is that the chi-squared divergence cannot be bounded by the KL divergence to leverage its chain rule as we did for the TV distance. In this sense, this alternate bound is not useful for elucidating the contributions of distribution shift in the local regret bound.

\section{Value of Information}\label{app:ValueOfInformation}

The local regret bound result in Theorem \ref{thm:SubMainTheorem} is useful to capture the value of information provided by the central server. Suppose a given client $c\in\calC$ has agency to decide whether to participate in the federated system including all other clients. If we consider the client as a local regret-minimizing agent, we can use the dominant terms in the appropriate local regret bounds to model the expected utility of the client. In particular, using prior results of standard offline policy learning \citep{zhou2023offline} and our findings in Theorems \ref{thm:MainTheorem} and \ref{thm:SubMainTheorem}, the value of information provided by the central server can be modeled as the comparison of the client's utility (as captured by the negative local regret) with and without participation
\begin{equation*}
    \calV_c(\w)=C_0\kappa(\Pi)\sqrt{V_c/n_c} - C_1\kappa(\Pi)\sqrt{V\skewness/n} - U\cdot\TV(\calDcec,\calDwew),
\end{equation*}
where $C_0,C_1$ are universal constants and $V_c=\supPi\E\regbra{\Gi(\pi(\Xi))^2}$ is the local AIPW variance.

Then, we say it is more valuable for client $c$ to participate in federation if $\calV_c(\w)>0$. One can easily show that this condition is satisfied if and only if
\begin{equation*}
    \TV(\calDcec,\calDwew) < \alpha r_c/U \quad\land\quad \skewness<\beta^2r_c^2/r^2
\end{equation*}
for some $\alpha+\beta\le 1$, where
\begin{math}
    r_c=C_0\kappa(\Pi)\sqrt{\smash[b]{V_c/n_c}}
\end{math}
is the local regret bound of the locally trained model and
\begin{math}
    r=C_1\kappa(\Pi)\sqrt{\smash[b]{V/n}}
\end{math}
is the global regret bound of the globally trained model under no skewness. The $\alpha$ and $\beta$ factors indicate a trade-off between distribution shift and skewness of the two conditions. If there is low distribution shift, the skewness can be large.

Thus, we see how the design choice on the client distribution must balance a scaled trade-off to achieve relative low skewness and relative low expected distribution shift. Indeed, in the experiments in Section \ref{sec:Experiments} we see how a skewed client distribution can help improve the local regret guarantees of a heterogeneous client. We observe from the first condition that the client distribution shift must must be smaller than the local regret of the locally trained model relative to a scaled range of the data. Intuitively, this states that the client will not benefit from federation if the regret they suffer due to their distribution shift from the global mixture distribution is greater than the relative local regret from just training locally. The second condition states that the skewness must be less than the local regret relative to the global regret. Large relative variance or small relative sample size are the primary factors that can lead small relative regret and therefore tight limitations on skewness budget. Overall, all of these conditions can be satisfied under sufficiently low expected distribution shift from the global distribution, low client distribution skewness, comparable AIPW variance across clients, and large global sample size relative to the local sample size.

However, it should be noted that this analysis simplifies the setting by considering only a single client that unilaterally decides to participate in the federation, without considering the choices of other clients. A more comprehensive analysis would assess the value of information provided by the central server in an equilibrium of clients with agency to participate. Game-theoretic aspects are crucial in this context, necessitating an understanding of client behavior and incentives in federated settings. Recent research has started to delve into game-theoretic considerations in federated supervised learning. \cite{donahue2021optimality} provided valuable insights into the behavior of self-interested, error-minimizing clients forming federated coalitions to learn supervised models. Moreover, designing incentive mechanisms in federated learning has been identified as a significant research area \citep{zhan2021survey}. This work aims to understand the optimal ways to incentivize clients to share their data. Applying these concepts to our setting would offer valuable insights on the incentives and behavior that motivate clients to participate in federated policy learning systems.

\section{Additional Algorithm Details}\label{app:AdditionalAlgorithmDetails}

\subsection{Nuisance Parameter Estimation}\label{app:NuisanceParameterEstimation}

Our results rely on efficient estimation of $\Qw(\pi)$ for any policy $\pi$, which in turn relies on efficient estimation of $\Qi(\pi)$. We leverage ideas of double machine learning \citep{chernozhukov2018double} to guarantee efficient policy value estimation given only high-level conditions on the predictive accuracy of machine learning methods on estimating the nuisance parameters of doubly robust policy value estimators. In this work, we use machine learning and cross-fitting strategies to estimate the nuisance parameters locally. The nuisance parameter estimates must satisfy the conditions of Assumption \ref{ass:FiniteSampleError}. Under these conditions, extensions of the results of \citep{chernozhukov2018double,athey2021policy} would imply that the doubly robust local policy value estimates $\Qhati(\pi)$ for any policy $\pi$ are asymptotically efficient for estimating $\Qi(\pi)$.

The conditions and estimators that guarantee these error assumptions have been extensively studied in the estimation literature. These include parametric or smoothness assumptions for non-parametric estimation. The conditional response function $\mui(x;a)=\E_{\calDcec}[\Yi(a)|\Xi=x]$ can be estimated by regressing observed rewards on observed contexts. The inverse conditional propensity function $\oi(x;a)=1/\P_{\calDcec}(\Ai=a|\Xi=x)$ can be estimated by estimating the conditional propensity function $\ei(x;a)=\P_{\calDcec}(\Ai=a|\Xi=x)$ and then taking the inverse. Under sufficient regularity and overlap assumptions, this gives accurate estimates. We can take any flexible approach to estimate these nuisance parameters. We could use standard parametric estimation methods like logistic regression and linear regression, or we could use non-parametric methods like classification and regression forests to make more conservative assumptions on the true models. Lastly, we note that if it is known that some clients have the same data-generating distribution, it should be possible to learn the nuisance parameters across similar clients.

In our experiments, we decided to estimate the $\mui$ with linear regression and $\ei$ with logistic regression. We used the sklearn Python package to fit the nuisance parameters. The true expected rewards are non-linear but the propensities are simple uniform probabilities. So our experiments emulate the scenario where accurate estimation of $\mui$ is not perfectly possible but accurate estimation of $\ei$ is easy, thus leveraging the properties of double machine learning for policy value estimation.

\subsection{Cross-fitted AIPW Estimation}\label{app:CAIPWEstimation}

Once the nuisance parameters are estimated, they can be used for estimating AIPW scores. Refer to Algorithm \ref{alg:CAIPW} for the pseudocode on how we conduct the cross-fitting strategy for AIPW score estimation. Under this strategy, each client $c\in\calC$ divides their local dataset into $K$ folds, and for each fold $k$, the client estimates $\mui$ and $\oi$ using the rest $K-1$ folds. During AIPW estimation for a single data point, the nuisance parameter estimate that is used in the AIPW estimate is the one that was not trained on the fold that contained that data point. This cross-fitting estimation strategy is described in additional detail in \citep{zhou2023offline}.

\begin{algorithm}
    \caption{Cross-fitted AIPW: Client-Side}
    \label{alg:CAIPW}
\begin{algorithmic}[1]
    \REQUIRE local data $\{(\Xij,\Aij,\Yij)\}_{i=1}^{\nc}$, nuisance parameter estimates $\muhati$ and $\ohati$, number of folds $K$
    \vspace{1pt}
    \STATE Partition local data into $K$ folds
    \STATE Define surjective mapping $k_c:[\nc]\to[K]$ of point index to corresponding fold index
    \FOR{$k=1,\dots,K$}
        \STATE Fit estimators $\muhati^{-k}$ and $\ohati^{-k}$ using rest of data not in fold $k$
    \ENDFOR

    \FOR{$i=1,\dots,\nc$}
        \FOR{$a\in\calA$}
            \STATE $\Ghatij(a)\leftarrow \muhatikij(\Xij;a)+\bigpar{\Yij-\muhatikij(\Xij;a)}\cdot\ohatikij(\Xij;a)\cdot\ones\{\Aij=a\}$
        \ENDFOR
    \ENDFOR
\end{algorithmic}
\end{algorithm}

\subsection{Implementation Details}

The local optimization problems we face in our formulation in Section \ref{sec:Algorithm} are equivalent to cost-sensitive multi-class classification (CSMC). There are many off-the-shelf methods available for such problem. We rely on implementations that can do fast online learning for parametric models in order to be able to do quick iterated updates on the global models at each local client and send these models for global aggregation. So we make use of the cost-sensitive one-against-all (CSOAA) implementation for cost-sensitive multi-class classification in the Vowpal Wabbit library \citep{vowpalwabbit}. This implementation performs separate online regressions of costs on contexts for each action using stochastic gradient descent updates.
At inference time, to make an action prediction, the action whose regressor gives the lowest predicted cost is chosen.

The idea behind this method is that if the classifiers admit regression functions that predict the costs, i.e., $\pi_\theta(x)=\argmax_{a\in\calA}f_\theta(x;a)$ for some $f_\theta\in\calF_\Theta$ such that $f^*(x;a)\in\calF_\Theta$ where $f^*(x;a)=\E[\smash[t]{\Gi(a)}|\Xi=x]$, then efficient regression oracles will return an (near) optimal model \citep{agarwal2017effective}. If realizability does not hold one may need to use more computationally expensive CSMC optimization techniques \citep{beygelzimer2009error}. For example, we could use the weighted all pairs (WAP) algorithm \citep{beygelzimer2008machine} that does $\binom{d}{2}$ pairwise binary classifications and predicts the action that receives majority predictions. Unlike the CSOAA implementation, the WAP method is always consistent in that an optimal model for the reduced problem leads to an optimal cost-sensitive prediction classifier. In our experiments, the rewards are non-linear so realizability does not exactly hold. Yet, we still observe good performance with the CSOAA regression-based algorithm.

\section{Additional Experimental Results}\label{app:AdditionalExperimentalResults}

We follow up on the simulations with heterogeneous clients in Section \ref{sec:Experiments}. Here, we observe the regret performance for one of the other clients that have less distribution shift from the average. Figure \ref{fig:image4} plots the local regret for client 2 of the globally trained policy (green) and the global regret of the globally trained policy (orange), all using the empirical mixture $\w=\barn$. For comparison, we also plot the local regret for client 2 of the locally trained policy (blue). The bands show the one standard deviation from the regrets over five different runs. As expected, we see that the other clients have less distribution shift so the local regret of the global policy nearly matches the global regret, similar to what was observed in the homogeneous experiments but with some level of degradation. Indeed, the local distributions nearly match the global distribution, by construction. In Figure \ref{fig:image5} we plot the same type of regret curves, but instead with the global policy trained with the skewed mixture. We see that their performance slightly degrades, in particular at the lower sample sizes where we are downscaling the local regret relatively more. This is in contrast to what we observe for client 1 where the skewed mixture improved performance. This is because, we are increasing distribution shift as measured by $\TV(\calDcec,\calDwew)$. This is another indicator that our theoretical regret guarantees may be tight.

\begin{figure}
    \centering
    \begin{subfigure}{0.49\linewidth}
        \centering
        \includegraphics[trim=10pt 0pt 50pt 10pt, clip, width=\linewidth]{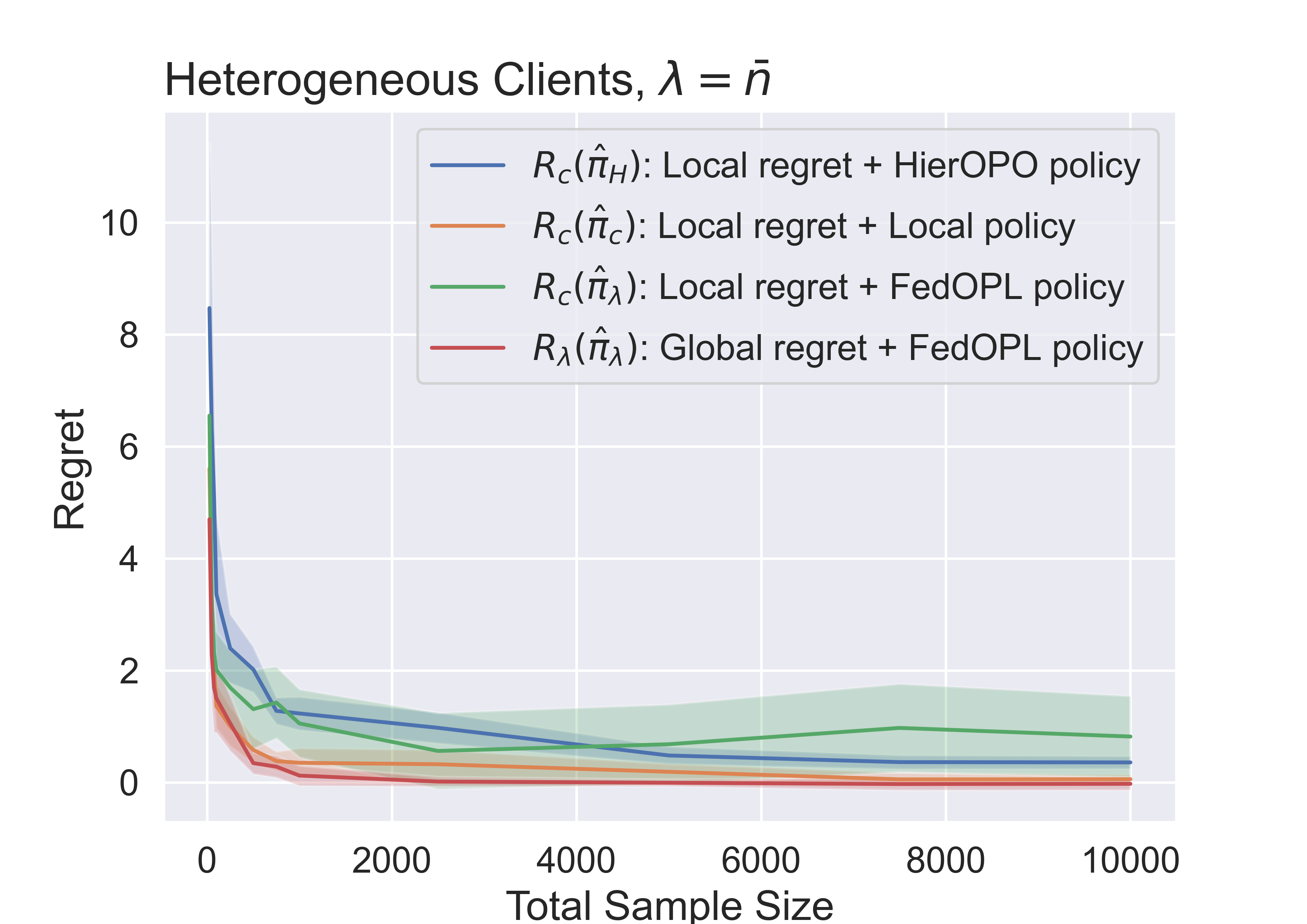}
        \caption{Heterogeneous Clients, $\w=\barn$}
        \label{fig:image4}
    \end{subfigure}
    \begin{subfigure}{0.49\linewidth}
        \centering
        \includegraphics[trim=10pt 0pt 50pt 10pt, clip, width=\linewidth]{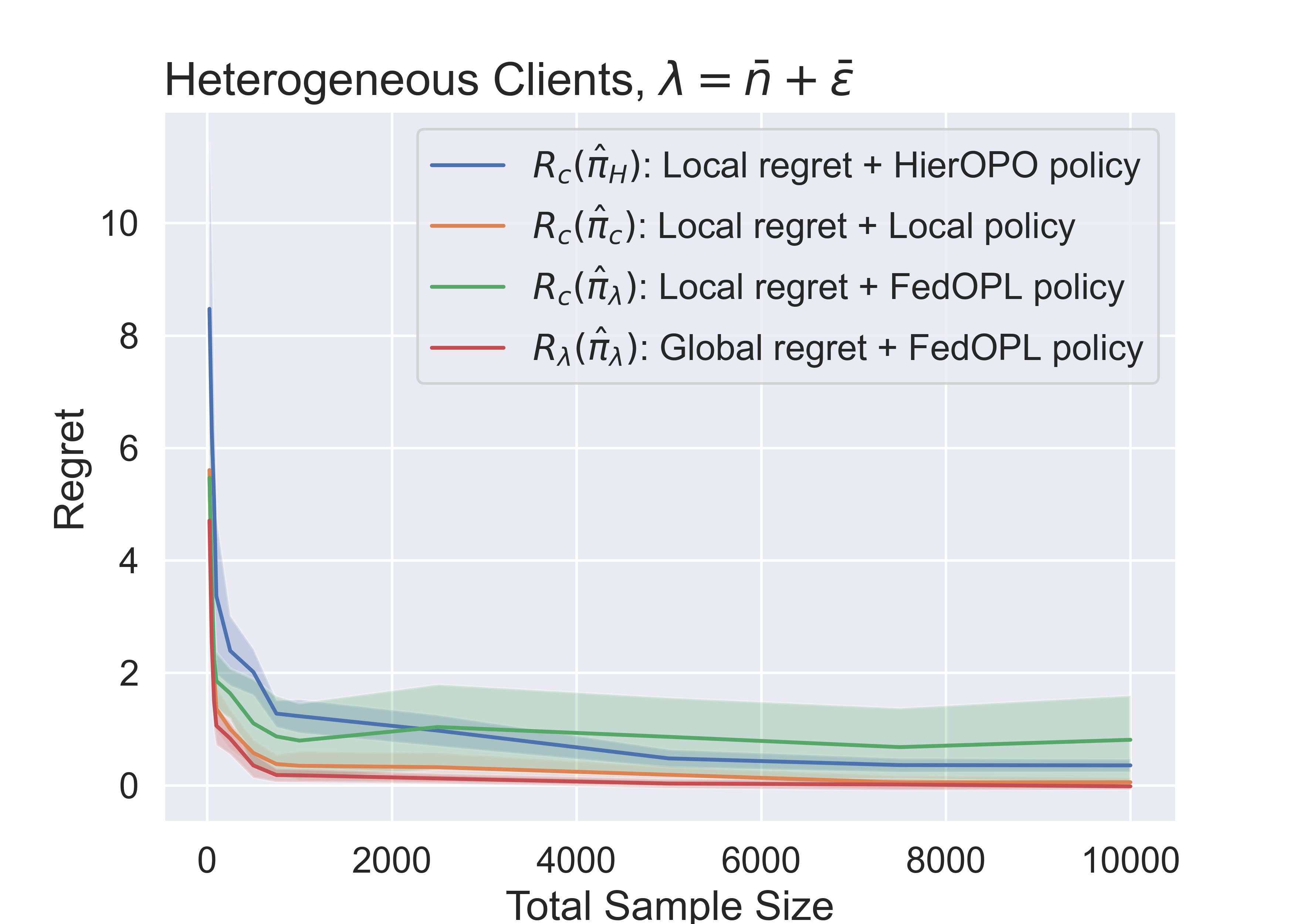}
        \caption{Heterogeneous Clients, $\w=\barn+\bar{\varepsilon}$}
        \label{fig:image5}
    \end{subfigure}
    \caption{Empirical regret curves for simulation experiments. Local regrets are for client 2.}
    \label{fig:client2_img}
\end{figure}

\section{Additional Discussion}\label{app:AdditionalDiscussion}

\subsection{Broader Impact}

The broader impact of this research lies in its potential to revolutionize how we learn and implement decision policies across diverse data sources. By enabling personalized policy learning from observational bandit feedback data, this work could significantly enhance the effectiveness of decision-making processes in various fields, including healthcare and social sciences. The ability to establish both global and local regret upper bounds provides a robust measure of the effectiveness of these policies, ensuring that they are beneficial on both an aggregate and individual level.

From an ethical standpoint, the federated setting of this research respects the privacy of individual data sources, as it does not require the collection of raw data. This approach aligns with the growing societal demand for data privacy and security. However, it also raises questions about the equitable distribution of benefits derived from shared data, especially when there are significant differences between data sources. As we move forward, it will be crucial to ensure that the benefits of such advanced policy learning methods are accessible to all participants, and that the trade-offs involved in the participation of heterogeneous data sources are transparent and fair. This research opens up new avenues for exploring these important issues, paving the way for a more inclusive and ethical data-driven future.

\subsection{Limitations \& Future Work}

There are several limitations to our present work that warrant further consideration, many of which are discussed throughout the main paper. However, here we provide a more exhaustive list of these limitations.

First, our work relies on certain assumptions about the data-generating process, which may not always hold. While we have discussed how some of these assumptions can potentially be relaxed, such as the boundedness and uniform overlap assumption in the data-generating distributions, further investigation is needed. An interesting question that arises from this is whether the pessimism principle, proposed by \cite{jin2022policy}, could be applied to overcome the uniform overlap assumption. Specifically, we must ask whether it is necessary to have coverage under the locally optimal policy for each data source, or if coverage under the globally optimal policy would suffice. Furthermore, the role that mixture weights play in satisfying this assumption should be examined.

Our approach also estimates nuisance parameters locally, which may be inefficient in some cases. For example, if it is known that some clients share the same data-generating distribution, it could be possible to learn these nuisance parameters across similar clients rather than estimating them individually. This presents an opportunity for improving the efficiency of our framework.

Additionally, while our framework represents a significant step toward privacy-preserving policy learning, it does not account for differential privacy considerations. Future work should explore the impact of differential privacy on our regret analysis and empirical results, as this would be highly relevant for practical applications.

Another limitation is the reliance of our optimization procedure on efficient online cost-sensitive classification methods. Although many fast implementations of these methods exist, they are often restricted to particular parametric classes. In policy learning scenarios, especially in public policy settings where decisions are subject to audits, simpler policy classes such as tree-based models are often preferred. Developing efficient federated learning algorithms for tree-based and other non-parametric policies would be a valuable extension of this work. More generally, further research is needed to develop federated methods that can accommodate a wider range of policy classes, including finite policies and neural policies.

We also assumed that the mixture distribution $\w$ is known in our analysis. A promising extension would be to address a more agnostic setting where the mixture distribution is optimized, perhaps using a minimax framework like the one proposed in \cite{mohri2019agnostic}. Such an approach could enhance the robustness and fairness of the global policy.

Moreover, the local regret bounds for each client include an irreducible term that arises from distribution shifts. It remains an open question whether this irreducible regret can be quantified in a federated manner, allowing the central server to determine whether any given client benefits from participation in federation.

In Section \ref{app:ValueOfInformation}, we discussed the value of information provided by the central server to an individual client in the context where all clients are assumed to participate in the federation. However, a more complete analysis would consider the value of information in an equilibrium where clients have the agency to choose whether to participate, rather than assuming all clients automatically engage in federation.

Finally, we leave open the question of whether the bounds we establish are regret optimal. As discussed in \cite{mohri2019agnostic}, skewness-based bounds for distributed supervised learning can be regret optimal. In the homogeneous setting, our results align with regret-optimal outcomes similar to those found in \cite{athey2021policy} and \cite{zhou2023offline}. While this suggests that our bounds may also be regret optimal, establishing lower bounds remains a topic for future research.


\end{document}